\documentclass{article} 
\usepackage{iclr2025_conference,times}

\usepackage{amsmath,amsfonts,bm}

\def\eqref#1{equation~\ref{#1}}

\def\1{\bm{1}}

\DeclareMathAlphabet{\mathsfit}{\encodingdefault}{\sfdefault}{m}{sl}
\SetMathAlphabet{\mathsfit}{bold}{\encodingdefault}{\sfdefault}{bx}{n}

\usepackage{hyperref}
\usepackage{url}

\usepackage{booktabs}
\usepackage{tabularx}
\usepackage{array}
\usepackage{makecell}
\usepackage[table]{xcolor}

\usepackage[table]{xcolor}
\usepackage{soul}
\usepackage{booktabs}
\usepackage{graphicx}
\usepackage{array}
\usepackage{amsmath}
\usepackage{amssymb}

\definecolor{topcolor}{HTML}{E7DCEF}
\definecolor{secondcolor}{HTML}{E5EDF3}

\usepackage{booktabs}  
\usepackage{tabularx}  
\usepackage{array}     

\usepackage{amsmath,amssymb,amsthm}
\newtheorem{proposition}{Proposition}

\usepackage{booktabs}
\usepackage{tabularx}
\usepackage{array}

\usepackage{algorithm}
\usepackage{algpseudocode}
\usepackage{xcolor}
\usepackage{amsmath,amssymb}
\usepackage{enumitem}
\usepackage{fontawesome5}
\usepackage{tikz}

\usepackage{colortbl}
\definecolor{gain}{RGB}{0,128,0}
\definecolor{seed}{RGB}{13,133,136}
\definecolor{purple}{RGB}{216, 110, 204}
\definecolor{resourceline}{HTML}{B8C9CF}
\definecolor{projectblue}{HTML}{2563EB}
\definecolor{githubblack}{HTML}{181717}
\definecolor{modelslate}{HTML}{475467}

\usepackage{hyperref}
\hypersetup{
    colorlinks=true,      
    citecolor=purple,       
}

\newcommand{\seedicon}{%
  \raisebox{-0.12em}{\includegraphics[height=0.99em]{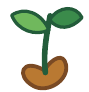}}
  \hspace{0.2em}%
}

\newcommand{\resourcepill}[4]{%
  \href{#1}{%
    \tikz[baseline=(resource.base)]{%
      \node[
        draw=resourceline,
        fill=white,
        rounded corners=8pt,
        line width=0.45pt,
        inner xsep=7pt,
        inner ysep=3pt
      ] (resource) {%
        \textcolor{#4}{#2}\hspace{0.4em}%
        \textcolor{black}{\sffamily\footnotesize\bfseries #3}%
      };%
    }%
  }%
}

\title{\texorpdfstring{\seedicon\textsc{\textcolor{seed}{Seed}}: \textcolor{seed}{Se}lf-\textcolor{seed}{E}volving On-Policy \textcolor{seed}{D}istillation for Agentic Reinforcement Learning}{SEED: Self-Evolving On-Policy Distillation for Agentic Reinforcement Learning}}

\author{\textbf{Jinyang Wu}$^{1}$\thanks{Equal Contribution}~~\thanks{Project Leader}~,
\textbf{Shuo Yang}$^{1}$\footnotemark[1]~~,
\textbf{Zhengxi Lu}$^{2}$, \textbf{Fan Zhang}$^{3}$, \textbf{Yuhao Shen}$^{2}$,
\textbf{Lang Feng}$^{4}$,\\
\textbf{Haoran Luo}$^{4}$,
\textbf{Zheng Lian}$^{5}$,
\textbf{Shuai Zhang}$^{1}$,
\textbf{Zhengqi Wen}$^{1}$,
\textbf{Jianhua Tao}$^{1}$\\
\\
$^1$Tsinghua University \qquad$^2$Zhejiang University \qquad$^3$The Chinese University of Hong Kong\\
$^4$Nanyang Technological University \qquad$^5$Tongji University\\
\texttt{Corresponding to: wu-jy23@mails.tsinghua.edu.cn} \\
\begin{tabular}{@{}ll@{}}
\end{tabular}}
\vspace{-0.1in}
\iclrfinalcopy 
\begin{document}

\maketitle

\vspace{-2.2em}
\begin{center}
  \resourcepill{https://jinyangwu.github.io/seed}{\faGlobe}{Project Page}{projectblue}
  \hspace{0.7em}
  \resourcepill{https://github.com/jinyangwu/SEED}{\faGithub}{Code}{githubblack}
  \hspace{0.7em}
  \resourcepill{https://huggingface.co/Jinyang23/Seed-AlfWorld-3B}{\faCube}{Model}{modelslate}
\end{center}
\vspace{0.35em}

\begin{abstract}
Large language models are increasingly trained as interactive agents for long-horizon tasks involving multi-turn interaction, tool use, and environment feedback. Outcome-based reinforcement learning (RL) provides a practical optimization paradigm, but its sparse trajectory-level rewards offer limited guidance on intermediate decisions, leaving a supervision gap between episode-level outcomes and token-level policy learning. We propose \textbf{\textsc{Seed}} (\textbf{SE}lf-\textbf{E}volving On-Policy \textbf{D}istillation), a self-evolving framework that converts completed on-policy trajectories into training-time hindsight skills and distills their behavioral effect back into the policy model. \textsc{Seed} first fine-tunes the policy to analyze completed trajectories and generate natural-language skills that capture reusable workflows, decisive observations, or failure-avoidance rules. During RL, the current policy both collects trajectories and serves as the analyzer that extracts hindsight skills from them. Policy updates therefore improve subsequent decision making and skill analysis together, allowing hindsight supervision to evolve with the policy. \textsc{Seed} then re-scores the sampled actions under ordinary and skill-augmented contexts, converting the skill-induced probability shift into a dense token-level on-policy distillation signal. This signal is jointly optimized with outcome-based RL, keeping the auxiliary supervision aligned with the current trajectory distribution. Extensive experiments on text-based and vision-based agentic tasks show that \textsc{Seed} consistently improves performance and sample efficiency, exhibiting robust generalization to unseen scenarios. Our code is available at \href{https://github.com/jinyangwu/SEED}{jinyangwu/\textsc{Seed}}.
\end{abstract}

\vspace{-0.1in}
\begin{figure}[h]
    \centering
    \includegraphics[width=0.99\linewidth]{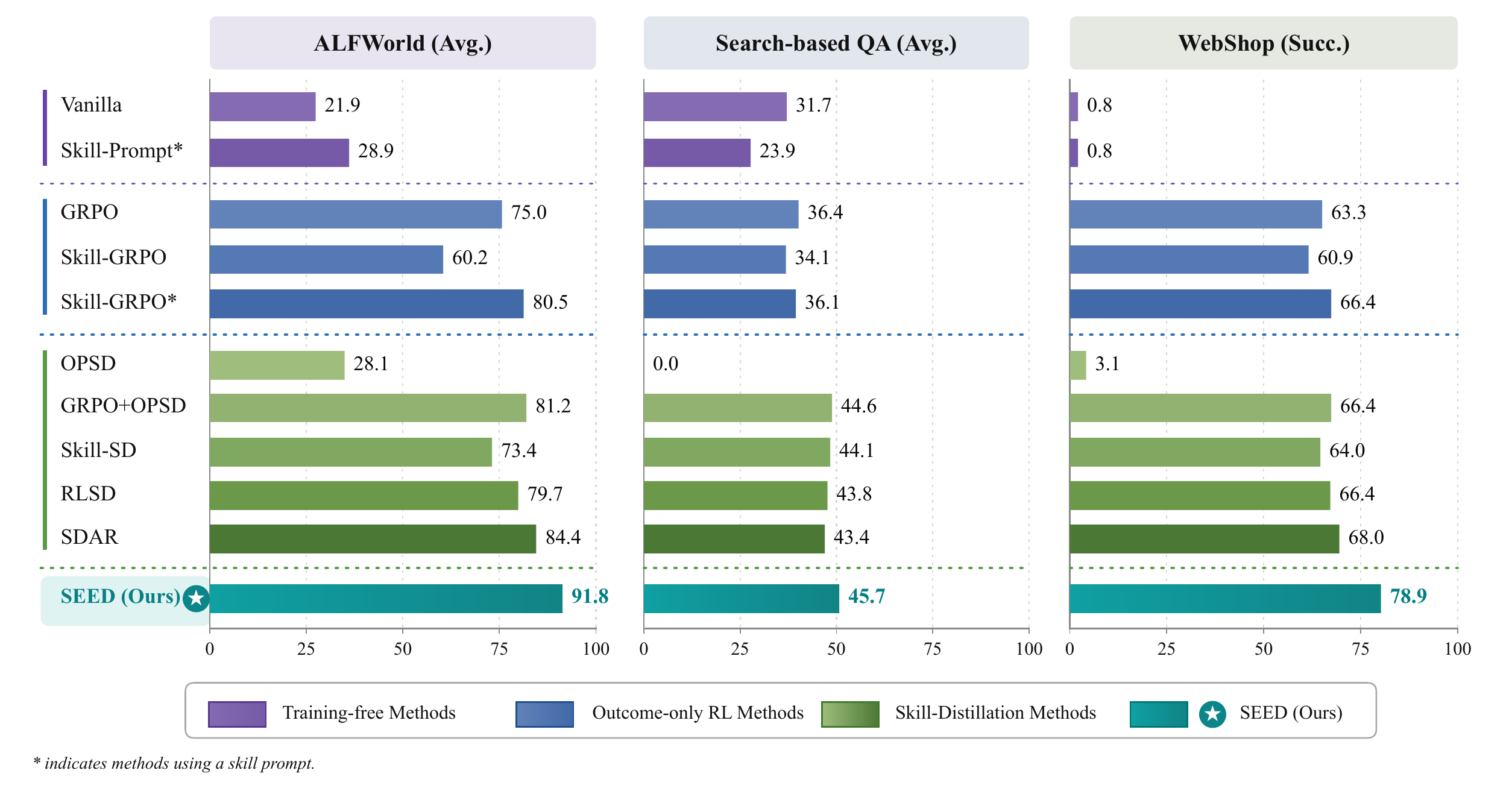}
    \caption{
        \textbf{Overall performance overview.}
        Compared with powerful baseline methods, \textsc{Seed} achieves the strongest average performance across three representative agentic benchmarks.
    }
    \label{fig:overall_performance}
\end{figure}

\section{Introduction}\label{sec1_intro}
Recent large language model (LLM) systems are moving beyond single-turn response generation toward multi-turn agentic interaction, where a model repeatedly reasons, acts, uses tools, and incorporates feedback from its environment~\citep{liu2023agentbench,schick2023toolformer,patil2024gorilla,luo2025large,xi2025rise,wu2026atlas,xu2026odysseyarena}. Such settings require an agent to make sequential decisions whose consequences may only become visible after many interaction steps. The model must learn when to gather information, when to call a tool, how to interpret feedback, and how to revise a plan after partial progress or failure. Reinforcement learning (RL) has therefore become an important post-training paradigm for LLM-based agents, since it directly optimizes policies against task-level feedback from environments, simulators, or verifiers~\citep{shao2024deepseekmath,wang2025ragen,luo2025agentlightning,wu2026spark}.

Despite this progress, outcome-based agentic RL provides only coarse supervision. In long-horizon environments, rewards are often sparse, delayed, and assigned at the trajectory level: they indicate whether an episode succeeds but not which intermediate observations, actions, or tool calls should be reinforced or corrected~\citep{andrychowicz2017hindsight,arjona2019rudder,uesato2022solving,lightman2024lets}. This leaves a supervision gap between episode-level outcomes and token-level policy learning. A failed trajectory may contain useful partial behaviors but fail because of a few local mistakes, whereas a successful trajectory may contain reusable strategies that the scalar reward never identifies. As a result, outcome-only optimization provides limited guidance for fine-grained, decision-level credit assignment over long interaction histories.

A key observation is that completed trajectories reveal hindsight unavailable during online decision making. Once an episode terminates, the full interaction history reveals which subgoals were achieved, where the agent deviated from an effective strategy, which observations were decisive, and which behavioral patterns may transfer to future attempts. This view is related to hindsight learning in RL, where completed experience can be reinterpreted to improve learning under sparse feedback~\citep{andrychowicz2017hindsight}, and to language-agent methods based on verbal reflection, episodic memory, or experience summaries~\citep{shinn2023reflexion,zhao2024expel,wang2023voyager,madaan2023selfrefine}. However, many such methods treat hindsight as static experience, inference-time context, or retrieved memory. For practical agentic RL, hindsight supervision should not remain fixed: as the policy improves and encounters new states, strategies, and failure modes, the hindsight extracted from its trajectories must adapt to its current behavior~\citep{andrychowicz2017hindsight,zhang2026the}. Our goal is therefore to convert policy-generated hindsight into parametric supervision, enabling the policy to internalize reusable behavioral guidance without external memory or additional deployment-time prompts.

On-policy distillation (OPD) offers a natural mechanism for converting hindsight information into decision-level learning signals. Classical knowledge distillation transfers teacher behavior into a student through token- or sequence-level supervision~\citep{hinton2015distilling,kim2016sequence}, while on-policy variants reduce distribution mismatch by supervising outputs sampled from the student policy itself~\citep{ross2011reduction,agarwal2024onpolicy}. Recent on-policy self-distillation methods further avoid a separate external teacher by comparing the same model under different contexts, such as privileged reasoning traces or feedback-conditioned prompts~\citep{zhao2026selfdistilledreasoner,hubotter2026sdpo}. Related agentic methods have explored skill- or feedback-conditioned distillation for multi-turn interaction and tool use~\citep{wang2026skillsd,lu2026sdar,zhong2026sod,ko2026reopold,yang2026opid}. Together, these advances point to three requirements for effective hindsight supervision in agentic RL.

First, the supervision should be \emph{on-policy}, because useful corrections depend on the states, actions, and failure modes induced by the current policy. Second, it should be \emph{dense}, allowing trajectory-level hindsight to guide individual decision tokens rather than only the final outcome. Third, it should be \emph{self-evolving}: as the policy improves, its decision-making and trajectory-analysis capabilities should advance together. Fixed teachers, static skill datasets, and one-time distillation cannot continually adapt to the policy's evolving capabilities and trajectory distribution.

We propose \textbf{\textsc{Seed}} (\textbf{SE}lf-\textbf{E}volving On-Policy \textbf{D}istillation), a framework that turns completed on-policy trajectories into hindsight skills and distills their behavioral effect back into the policy model through a self-evolving training loop. During RL, the current policy collects trajectories and also serves as the analyzer that extracts natural-language skills describing reusable workflows, decisive observations, or failure-avoidance rules. Because both roles share the same model, policy updates improve subsequent decision making and skill analysis together, allowing hindsight supervision to evolve with the policy. Given an extracted skill, \textsc{Seed} keeps the sampled actions fixed and re-scores them under ordinary and skill-augmented contexts. The resulting probability shift provides a dense token-level OPD signal, which is jointly optimized with the outcome-based RL objective. Specifically, \textsc{Seed} consists of two stages. First, \textbf{hindsight-skill supervised fine-tuning} equips the model to analyze completed interaction histories and generate reusable trajectory-level skills. Second, \textbf{self-evolving OPD} repeatedly uses the latest policy checkpoint for both trajectory collection and skill analysis, then updates the policy with the joint RL and OPD objective. The generated skills act only as privileged supervision during training and require neither external memory nor additional prompts at inference time.

We evaluate \textsc{Seed} across embodied interaction, web navigation, search-based QA, and visual perception and planning. \textsc{Seed} achieves superior task performance, sample efficiency, and robustness. Taken together, our work makes the following contributions:
\begin{itemize}[leftmargin=1.5em]
    \item We propose \textbf{\textsc{Seed}}, a self-evolving OPD framework that continually transforms the policy's completed trajectories into hindsight skills and internalizes their behavioral guidance during agentic RL, allowing decision-making and skill analysis to improve together.

    \item We introduce a policy-synchronized hindsight OPD mechanism that converts skill-induced log-probability shifts on sampled actions into dense token-level supervision and jointly optimizes this signal with outcome-based RL.
    
    \item Extensive experiments across diverse long-horizon agentic benchmarks show that \textsc{Seed} improves task performance, sample efficiency, and robustness over representative baselines.
\end{itemize}

\section{Related Work}\label{sec:related_work}

\paragraph{Reinforcement learning for agentic LLMs.}
LLMs are increasingly trained as interactive agents that reason, use tools, and act over long horizons~\citep{liu2023agentbench,schick2023toolformer,patil2024gorilla,luo2025large,xi2025rise,wu2026atlas,xu2026odysseyarena}. Reinforcement learning provides a natural post-training paradigm for such agents by directly optimizing task-level outcomes~\citep{shao2024deepseekmath,wang2025ragen,luo2025agentlightning,wu2026spark,lu2026skill0}. However, outcome-based RL typically assigns sparse and delayed rewards at the trajectory level, offering limited guidance on which intermediate observations, actions, or tool calls should be reinforced. SEED retains outcome-based RL as the optimization backbone while supplementing it with dense supervision derived from completed trajectories.

\vspace{-0.08in}
\paragraph{Hindsight learning for language agents.}
Completed trajectories expose reusable strategies, decisive observations, and failure causes that are unavailable during online decision making. Prior work exploits such information through hindsight relabeling, return decomposition, process supervision, verbal reflection, and experience memory~\citep{andrychowicz2017hindsight,arjona2019rudder,uesato2022solving,lightman2024lets,shinn2023reflexion,zhao2024expel,wang2023voyager,madaan2023selfrefine}. Nevertheless, hindsight knowledge is often stored as static experience or introduced as additional inference-time context. SEED instead converts completed on-policy trajectories into natural-language skills and internalizes their behavioral guidance into the policy, requiring neither external memory nor skill prompts at inference time.

\vspace{-0.08in}
\paragraph{On-policy self-distillation for agentic RL.}
Knowledge distillation transfers teacher behavior through token- or sequence-level supervision, while on-policy variants reduce distribution mismatch by training on outputs sampled from the learner itself~\citep{hinton2015distilling,kim2016sequence,ross2011reduction,agarwal2024onpolicy}. Recent methods construct privileged self-teachers from reasoning traces, feedback, or skills and use them to provide token-level guidance for agentic RL~\citep{zhao2026selfdistilledreasoner,wang2026skillsd,lu2026sdar,zhong2026sod,yang2026opid}. However, the source of privileged supervision is often static, externally generated, or updated independently of the policy. As the policy improves and encounters new states and failure modes, such supervision can become stale or mismatched with its current behavior. SEED addresses this limitation through a self-evolving loop in which the latest policy checkpoint serves simultaneously as the rollout actor and the trajectory analyzer. After each policy update, the shared checkpoint is refreshed for both roles, allowing decision making and hindsight supervision to co-evolve.

\vspace{-0.1in}
\begin{figure}[t]
    \centering
    \includegraphics[width=0.96\linewidth]{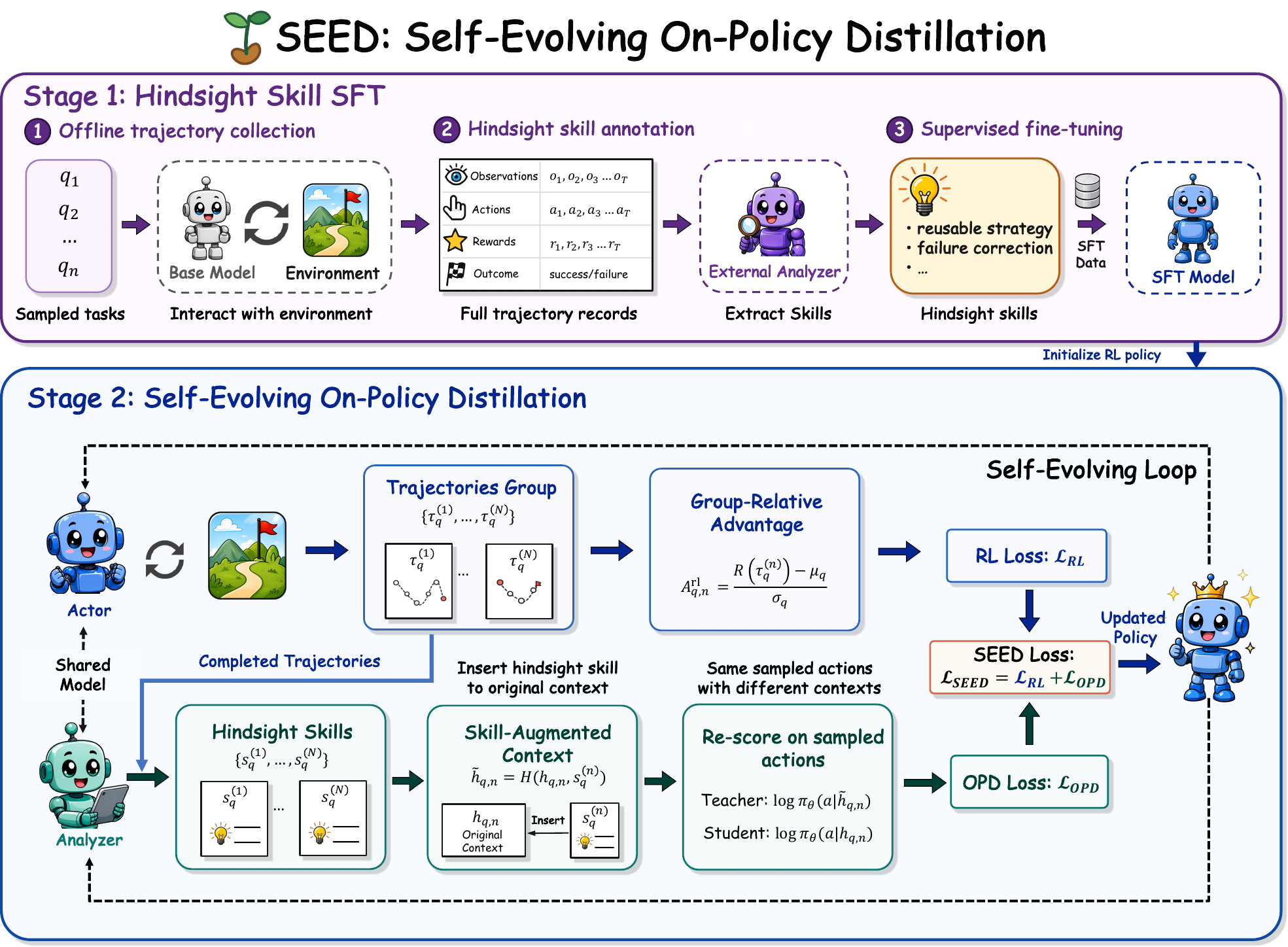}
    \caption{\textbf{Overview of \textsc{Seed}.} \textbf{Stage 1} (Hindsight Skill SFT) equips the policy to extract hindsight skills from completed trajectories. \textbf{Stage 2} (Self-Evolving On-Policy Distillation) jointly optimizes outcome-based RL and skill-conditioned OPD in a self-evolving agentic loop.}
    \label{fig:overview}
    \vspace{-0.1in}
\end{figure}

\section{Method}
We present \textbf{\textsc{Seed}}, a self-evolving OPD framework for agentic RL. \textsc{Seed} is motivated by the observation that completed agent trajectories contain rich hindsight information: even when reward is sparse, a full trajectory often reveals useful behavioral patterns, failure causes, and reusable strategies that are not directly available at intermediate decision steps. \textsc{Seed} converts such hindsight information into \textbf{hindsight skills} and distills their behavioral effect back into the ordinary policy.

As shown in Figure~\ref{fig:overview}, \textsc{Seed} has two training stages. First, hindsight-skill supervised fine-tuning (SFT) equips a single policy model to analyze completed trajectories. Second, the current policy snapshot both collects on-policy trajectories and analyzes them into hindsight skills. The same sampled actions are then re-scored under ordinary and skill-augmented contexts to construct a token-level OPD signal, which is optimized jointly with RL. At inference time, \textsc{Seed} removes the analyzer, so deployment requires only the learned policy.

\subsection{Problem Formulation}
We formulate long-horizon agentic tasks as partially observable Markov decision processes:
\[
\left(
\mathcal{S},
\mathcal{A},
\mathcal{O},
\mathcal{T},
\Omega,
\mathcal{R},
\gamma
\right),
\]
where $\mathcal{S}$ is the latent state space, $\mathcal{A}$ is the action space, $\mathcal{O}$ is the observation space, $\mathcal{T}$ is the transition kernel, $\Omega$ is the observation kernel, $\mathcal{R}$ is the reward function, and $\gamma$ is the discount factor. At timestep $t$, the agent receives an observation $o_t \in \mathcal{O}$ and maintains an interaction history
\[
h_t =
\left(
o_0, a_0, o_1, a_1, \ldots, o_t
\right),
\]
where $a_i$ denotes the textual response or executable action produced by the agent. The policy generates the next action according to
\[
a_t
\sim
\pi_\theta
\left(
\cdot \mid h_t
\right).
\]

A completed trajectory is denoted by
\[
\tau =
\left\{
(o_t, a_t, r_t)
\right\}_{t=0}^{T-1},
\]
with an episode-level outcome $R(\tau) \in \mathbb{R}$. In many agentic environments, $R(\tau)$ is sparse and becomes available only after task completion, such as a binary success indicator or a final task score. The standard reinforcement learning objective is
\[
J(\theta)
=
\mathbb{E}_{\tau \sim \pi_\theta}
\left[
R(\tau)
\right].
\]
This formulation exposes a key challenge in long-horizon agentic RL: the environment provides only trajectory-level feedback, while the policy must learn from many token-level decisions distributed across the interaction history. \textsc{Seed} bridges this granularity gap by deriving an auxiliary training-time token-level signal from hindsight skills extracted after trajectory completion.

\subsection{Hindsight Skill Supervised Fine-Tuning}
The first stage initializes the policy with the ability to analyze full interaction histories and express reusable behavioral guidance as natural-language hindsight skills.

\paragraph{Offline trajectory collection.}
We first collect an offline pool of trajectories using a base policy without skill augmentation. Let $\mathcal{Q}_{\mathrm{sft}} = \{q_j\}_{j=1}^{M}$ be a set of training tasks. For each task $q_j$, we run $K_0$ rollouts with the base policy $\pi_{\theta_{\mathrm{base}}}$:
\[
\mathcal{B}_j
=
\left\{
\tau_{j,k}
\right\}_{k=1}^{K_0},
\qquad
\tau_{j,k}
\sim
\pi_{\theta_{\mathrm{base}}}
\left(
\cdot \mid q_j
\right).
\]
The complete offline trajectory pool is
\[
\mathcal{B}
=
\bigcup_{j=1}^{M}
\mathcal{B}_j.
\]
Each trajectory contains the task description, observations, actions, rewards, and final outcome. These trajectories are collected without any hindsight skill in the decision context, ensuring that the SFT data are derived from ordinary agent-environment interaction.

\paragraph{Hindsight skill annotation.}
Given a completed trajectory $\tau$, an external analyzer $A_{\mathrm{ext}}$ produces a hindsight-skill annotation:
\[
s_\tau
=
A_{\mathrm{ext}}
\left(
\tau
\right),
\]
where $s_\tau$ denotes the skill annotation produced from trajectory $\tau$. For a successful trajectory, $s_{\tau}$ typically captures reusable strategies or workflows that contributed to task completion. For a failed trajectory, $s_\tau$ can encode corrective or avoidance guidance inferred from the observed failure.

We retain the annotation only if the generated skill is correctly formatted. Let $v_\tau \in \{0,1\}$ denote the validity indicator. The accepted SFT set is
\[
\mathcal{D}_{\mathrm{sft}}
=
\left\{
(x_\tau, s_\tau)
:
\tau \in \mathcal{B},
v_\tau = 1
\right\},
\]
where $x_\tau$ is the serialized trajectory-analysis input and $s_\tau$ is the corresponding hindsight skill target.

\paragraph{Supervised fine-tuning.}
We then fine-tune the policy model to predict the hindsight skill from the completed trajectory. The same autoregressive model later initializes both the RL actor and the synchronized trajectory analyzer. For an accepted SFT example $(x_\tau, s_\tau) \in \mathcal{D}_{\mathrm{sft}}$, the model is optimized with the standard negative log-likelihood objective:
\[
\mathcal{L}_{\mathrm{sft}}(\theta)
=
-
\mathbb{E}_{(x_\tau,s_\tau)\sim\mathcal{D}_{\mathrm{sft}}}
\left[
\sum_{\ell=1}^{|s_\tau|}
\log \pi_\theta\!\left(
s_{\tau,\ell}
\mid
x_\tau, s_{\tau,<\ell}
\right)
\right],
\]
where $s_{\tau,\ell}$ is the $\ell$-th token of $s_\tau$. The resulting checkpoint $\theta_{\mathrm{sft}}$ initializes the later RL policy.

During RL, the trajectory analyzer is instantiated directly from the current policy checkpoint. Thus, the same model can act in the environment under ordinary interaction histories and generate hindsight skills from completed trajectories, without a separately trained analyzer.

\subsection{Self-Evolving On-Policy Distillation}
The second stage performs agentic RL with an additional token-level OPD signal. At the start of each update, \textsc{Seed} freezes the current policy as $\pi_{\theta_{\mathrm{old}}}$. This snapshot collects trajectories and also parameterizes the analyzer that extracts their hindsight skills. A trainable policy $\pi_\theta$, initialized from $\pi_{\theta_{\mathrm{old}}}$, is then optimized with both the environment-driven GRPO and OPD objectives. After the update, the optimized policy $\pi_\theta$ becomes the policy snapshot for the next iteration.

This design keeps the distillation signal on-policy: the policy collects its own trajectories, the analyzer agent converts these completed trajectories into hindsight skills, and the behavioral effect of such hindsight guidance is distilled back into the ordinary policy.

\paragraph{On-policy hindsight skill generation.}
For each task prompt $q$, \textsc{Seed} samples a group of $N$ trajectories using the frozen policy:
\[
\mathcal{G}_q
=
\left\{
\tau_q^{(1)}, \tau_q^{(2)}, \ldots, \tau_q^{(N)}
\right\},
\qquad
\tau_q^{(n)}
\sim
\pi_{\theta_{\mathrm{old}}}
\left(
\cdot \mid q
\right).
\]
For each completed trajectory $\tau_q^{(n)}$, \textsc{Seed} constructs its trajectory-analysis input $x_{\tau_q^{(n)}}$. The analyzer agent instantiated from the same policy snapshot $\pi_{\theta_{\mathrm{old}}}$ then analyzes the completed trajectory and generates a hindsight skill:
\[
s_q^{(n)}
=
A_{\theta_{\mathrm{old}}}
\left(
x_{\tau_q^{(n)}}
\right),
\]
where $A_{\theta_{\mathrm{old}}}$ denotes the analyzer role of the shared model. Although the actor and analyzer share the same parameters at each update, they play different roles: the actor interacts with the environment, while the analyzer summarizes completed trajectories into trajectory-level hindsight skills. The generated skill $s_q^{(n)}$ thus provides reusable behavioral guidance.

This shared parameterization creates \textsc{Seed}'s self-evolving loop. Refreshing $\theta_{\mathrm{old}}$ changes both the trajectories encountered by the actor and the model capability used for skill analysis. Consequently, the experience distribution and its hindsight supervision evolve together.

\paragraph{On-policy distillation objective.}
\textsc{Seed} keeps the original on-policy actions fixed and re-scores them under a skill-augmented context. Let $H$ be a deterministic context augmentation function that incorporates the generated skill into the ordinary interaction history. At timestep $t$ in trajectory $\tau_q^{(n)}$, the skill-augmented history is
\[
\tilde{h}_{q,n,t}
=
H
\left(
h_{q,n,t},
s_q^{(n)}
\right).
\]
The original sampled action is tokenized as
\[
a_{q,n,t}
=
\left(
a_{q,n,t,1},
\ldots,
a_{q,n,t,L_{q,n,t}}
\right),
\]
where $L_{q,n,t}$ denotes the number of tokens in the sampled action $a_{q,n,t}$.

The same policy computes two token-level log-probabilities on the same sampled action tokens. The first is the skill-conditioned teacher branch log-probability, obtained by re-scoring the sampled action under the skill-augmented history:
\[
\ell^{\mathrm{skill}}_{q,n,t,\ell}
=
\log
\pi_{\theta}
\left(
a_{q,n,t,\ell}
\mid
\tilde{h}_{q,n,t},
a_{q,n,t,<\ell}
\right).
\]
The second is the ordinary student branch log-probability under the original interaction history:
\[
\ell^{\theta}_{q,n,t,\ell}
=
\log
\pi_{\theta}
\left(
a_{q,n,t,\ell}
\mid
h_{q,n,t},
a_{q,n,t,<\ell}
\right).
\]
Although both branches share $\pi_\theta$, they correspond to different input contexts: the teacher observes the hindsight skill, while the student acts only from the ordinary history. The teacher provides a detached training-time signal, and gradients flow exclusively through the ordinary student branch. In summary, during optimization, both branches are evaluated using the current trainable policy $\pi_\theta$, whereas $\pi_{\theta_{\mathrm{old}}}$ is used for trajectory collection, skill generation, and importance-ratio computation.

We define the detached skill-induced log-probability shift
\[
\Delta_{q,n,t,\ell}
=
\operatorname{sg}
\left[
\ell^{\mathrm{skill}}_{q,n,t,\ell}
-
\ell^{\theta}_{q,n,t,\ell}
\right],
\]
where $\operatorname{sg}[\cdot]$ denotes stop-gradient. Following SDAR~\citep{lu2026sdar}, SEED maps this shift to a confidence gate:
\[
g_{q,n,t,\ell}
=
\sigma
\left(
\beta_{\mathrm{opd}}
\Delta_{q,n,t,\ell}
\right),
\]
where $\sigma(\cdot)$ is the logistic sigmoid function and $\beta_{\mathrm{opd}}$ controls the sharpness of the gate. A positive shift indicates that the hindsight skill supports the sampled token and yields a larger gate; a negative shift attenuates that token's auxiliary supervision.

The OPD loss is defined as a confidence-gated sampled-token distillation objective:
\begin{equation}
\mathcal{L}_{\mathrm{opd}}(\theta)
=
\mathbb{E}_{q,n,t,\ell}
\left[
m_{q,n,t,\ell}
\cdot
g_{q,n,t,\ell}
\cdot
\left(
\operatorname{sg}
\left[
\ell^{\mathrm{skill}}_{q,n,t,\ell}
\right]
-
\ell^{\theta}_{q,n,t,\ell}
\right)
\right].
\label{eq:opd_objective}
\end{equation}
Here $m_{q,n,t,\ell} \in \{0,1\}$ is the valid-token mask. Throughout, expectations over $(q,n,t,\ell)$ are computed as masked means over valid action tokens, normalized by $\sum_{q,n,t,\ell}m_{q,n,t,\ell}$. Because both $g_{q,n,t,\ell}$ and the teacher log-probability are detached, the teacher term is constant with respect to $\theta$, and
\[
\nabla_\theta \mathcal{L}_{\mathrm{opd}}
=
-
\mathbb{E}_{q,n,t,\ell}
\left[
m_{q,n,t,\ell}
\cdot
g_{q,n,t,\ell}
\cdot
\nabla_\theta
\ell^{\theta}_{q,n,t,\ell}
\right].
\]
Thus, the objective is gradient-equivalent to a gate-weighted negative log-likelihood. Minimizing it increases the ordinary policy's likelihood of teacher-endorsed on-policy tokens, thereby internalizing the skill's behavioral effect without exposing the skill at inference time.

\paragraph{Joint training objective.}
In addition to OPD, \textsc{Seed} optimizes the policy with a group-relative RL objective. For each group $\mathcal{G}_q$, we compute the mean and standard deviation of trajectory outcomes:
\[
\mu_q
=
\frac{1}{N}
\sum_{n=1}^{N}
R
\left(
\tau_q^{(n)}
\right),
\quad
\sigma_q
=
\sqrt{
\frac{1}{N}
\sum_{n=1}^{N}
\left(
R
\left(
\tau_q^{(n)}
\right)
-
\mu_q
\right)^2
}.
\]
The trajectory-level group-relative advantage is
\[
A^{\mathrm{rl}}_{q,n}
=
\frac{
R
\left(
\tau_q^{(n)}
\right)
-
\mu_q
}{
\sigma_q+\epsilon
}.
\]
This advantage is broadcast to valid action tokens:
\[
A^{\mathrm{rl}}_{q,n,t,\ell}
=
A^{\mathrm{rl}}_{q,n}
m_{q,n,t,\ell}.
\]

The token-level probability ratio is
\[
\rho_{q,n,t,\ell}(\theta)
=
\exp
\left(
\ell^\theta_{q,n,t,\ell}
-
\ell^{\mathrm{old}}_{q,n,t,\ell}
\right),
\]
where
\[
\ell^{\mathrm{old}}_{q,n,t,\ell}
=
\log
\pi_{\theta_{\mathrm{old}}}
\left(
a_{q,n,t,\ell}
\mid
h_{q,n,t},
a_{q,n,t,<\ell}
\right).
\]
The RL loss is
\[
\begin{aligned}
\mathcal{L}_{\mathrm{rl}}(\theta)
&=
-
\mathbb{E}_{q,n,t,\ell}
\left[
\min
\left(
\rho_{q,n,t,\ell}(\theta)
A^{\mathrm{rl}}_{q,n,t,\ell},
\operatorname{clip}
\left(
\rho_{q,n,t,\ell}(\theta),
1-\epsilon_{\mathrm{clip}},
1+\epsilon_{\mathrm{clip}}
\right)
A^{\mathrm{rl}}_{q,n,t,\ell}
\right)
\right]
\\
&\quad
+
\beta_{\mathrm{KL}}
D_{\mathrm{KL}} .
\end{aligned}
\label{eq:rl_objective}
\]
where $\beta_{\mathrm{KL}}$ is the KL regularization coefficient.

The final training objective combines environment-driven agentic RL with hindsight-skill OPD:
\[
\mathcal{L}_{\mathrm{SEED}}(\theta)
=
\mathcal{L}_{\mathrm{rl}}(\theta)
+
\lambda_{\mathrm{opd}}
\mathcal{L}_{\mathrm{opd}}(\theta),
\]
where $\lambda_{\mathrm{opd}}$ controls the strength of the auxiliary OPD signal. The RL term optimizes environment outcomes, whereas the OPD loss internalizes the effect of self-generated hindsight skills. The updated policy then becomes $\pi_{\theta_{\mathrm{old}}}$ for the next iteration, closing the self-evolving loop.

At inference time, the deployed agent acts only from the ordinary interaction history:
\[
a_t
\sim
\pi_\theta
\left(
\cdot
\mid
h_t
\right).
\]
Thus, all hindsight skills are used only as training-time guidance. Deployment requires no analyzer, no skill bank, no retrieval module, and no augmented decision prompt.

Algorithm~\ref{alg:seed} in Appendix~\ref{app:algorithm_skills} summarizes the complete training procedure.

\section{Experiment}
\subsection{Experimental Setting}

\paragraph{Benchmarks.}
We evaluate \textsc{Seed} across three complementary forms of long-horizon agency. \textit{ALFWorld}~\citep{shridhar2021alfworld} casts household tasks as text-based embodied interaction, requiring an agent to interpret observations and execute extended action sequences. We consider six task families: \textit{Pick}, \textit{Look}, \textit{Clean}, \textit{Heat}, \textit{Cool}, and \textit{Pick2}. \textit{WebShop}~\citep{yao2022webshop} evaluates interactive web navigation, where an agent searches for products, inspects their attributes, and completes a purchase according to a natural-language request. We use the standard set of 128 test tasks. Finally, following the Search-R1 protocol~\citep{jin2025searchr1}, \textit{Search-based QA} requires an agent to gather evidence through search before answering questions from Natural Questions~\citep{kwiatkowski2019natural}, TriviaQA~\citep{joshi2017triviaqa}, PopQA~\citep{mallen2023popqa}, HotpotQA~\citep{yang2018hotpotqa}, 2WikiMultiHopQA~\citep{ho2020constructing}, MuSiQue~\citep{trivedi2022musique}, and Bamboogle~\citep{press2023measuring}. Together, these benchmarks cover embodied control, web-based decision making, and tool-augmented information seeking.

\vspace{-0.08in}

\paragraph{Baselines.}
We compare \textsc{Seed} with prompting, outcome-based RL, and distillation baselines. \textit{Vanilla} evaluates the instruction-tuned backbone without post-training, whereas \textit{Skill-Prompt} provides natural-language skills only in the evaluation context. \textit{GRPO}~\citep{shao2024deepseekmath} optimizes group-normalized trajectory rewards without auxiliary supervision. \textit{Skill-GRPO} additionally conditions the policy on skills during RL. We also include representative self-distillation and skill-distillation methods: \textit{OPSD}~\citep{zhao2026selfdistilledreasoner}, \textit{GRPO+OPSD}, \textit{Skill-SD}~\citep{wang2026skillsd}, \textit{RLSD}~\citep{yang2026rlsd}, and \textit{SDAR}~\citep{lu2026sdar}. These comparisons distinguish the effects of outcome optimization, access to skill context, and dense teacher-derived supervision. An asterisk denotes evaluation with skills; all other methods operate from the ordinary interaction history without privileged skill inputs at test time. We use matched backbones, rollout budgets, and training schedules for all reproduced post-training baselines.

\vspace{-0.08in}

\begin{table*}[ht!]
    \centering
    \caption{
        \textbf{Performance Comparison on the representative long-horizon benchmarks (ALFWorld, Search-based QA, and WebShop).}
        We report the success rate (\%) on ALFWorld, accuracy on Search-based QA, and task-completion score/success rate on WebShop. An asterisk (*) denotes validation with skills. The \sethlcolor{topcolor}\hl{\textbf{best}} and \sethlcolor{secondcolor}\hl{\mbox{\underline{second-best}}} results are highlighted.
    }
    \vspace{0.05in}
    \label{tab:main_results}
    \resizebox{\textwidth}{!}{%
    \begin{tabular}{l ccccccc cccccccc cc}
    \toprule
    & \multicolumn{7}{c}{\textbf{ALFWorld}}
    & \multicolumn{8}{c}{\textbf{Search-based QA}}
    & \multicolumn{2}{c}{\textbf{WebShop}} \\
    \cmidrule(lr){2-8}
    \cmidrule(lr){9-16}
    \cmidrule(lr){17-18}
    \textbf{Method}
    & \textbf{Pick} & \textbf{Look} & \textbf{Clean}
    & \textbf{Heat} & \textbf{Cool} & \textbf{Pick2} & \textbf{Avg}
    & \textbf{NQ} & \textbf{Triv} & \textbf{Pop}
    & \textbf{Hotp} & \textbf{2Wk} & \textbf{MuS}
    & \textbf{Bam} & \textbf{Avg}
    & \textbf{Score} & \textbf{Succ.} \\
    \midrule

    \rowcolor{gray!10}
    \multicolumn{18}{c}{\textit{Qwen2.5-3B-Instruct}} \\

    Vanilla
        & 44.4 & 11.1 & 6.2 & 15.4 & 28.6 & 12.5 & 21.9
        & 24.6 & 48.1 & 31.0 & 26.3 & 25.3 & 7.2 & 59.7 & 31.7
        & 6.7 & 0.8
        \\
    Skill-Prompt*
        & 51.7 & 66.7 & 48.4 & 0.0 & 4.3 & 10.0 & 28.9
        & 23.7 & 46.2 & 30.6 & 24.4 & 22.1 & 7.5 & 12.5 & 23.9
        & 0.2 & 0.8
        \\
    OPSD
        & 48.8 & 41.7 & 16.7 & 0.0 & 15.8 & 16.7 & 28.1
        & 0.1 & 0.1 & 0.1 & 0.0 & 0.0 & 0.0 & 0.0 & 0.0
        & 11.3 & 3.1
        \\
    GRPO
        & 91.2 & 62.5 & \cellcolor{secondcolor}\underline{96.2}
        & 61.9 & 65.0 & 47.4 & 75.0
        & 39.3 & 60.6 & 41.1 & 37.4 & 34.6 & 15.4 & 26.4 & 36.4
        & 79.8 & 63.3
        \\
    Skill-GRPO
        & 88.9 & 71.4 & 58.8 & 70.6 & 40.7 & 29.2 & 60.2
        & 43.5 & 58.8 & 43.0 & 36.8 & 32.2 & 11.7 & 12.5 & 34.1
        & 77.3 & 60.9
        \\
    Skill-GRPO*
        & 94.3 & 57.1 & \cellcolor{topcolor}\textbf{100.0}
        & 66.7 & \cellcolor{secondcolor}\underline{73.1}
        & 57.1 & 80.5
        & 44.3 & 59.6 & 44.3 & 39.0 & 36.1 & 14.5 & 14.9 & 36.1
        & 76.3 & 66.4
        \\
    GRPO+OPSD
        & \cellcolor{topcolor}\textbf{100.0}
        & \cellcolor{secondcolor}\underline{82.4}
        & 85.7
        & \cellcolor{secondcolor}\underline{75.0}
        & 70.0 & 60.0 & 81.2
        & \cellcolor{topcolor}\textbf{44.9}
        & \cellcolor{secondcolor}\underline{61.2}
        & \cellcolor{secondcolor}\underline{45.2}
        & \cellcolor{secondcolor}\underline{40.4}
        & 38.5 & 16.0 & 66.1
        & \cellcolor{secondcolor}\underline{44.6}
        & 77.8 & 66.4
        \\
    Skill-SD
        & 88.2 & 50.0 & \cellcolor{secondcolor}\underline{96.2}
        & 52.4 & 65.0 & 57.9 & 73.4
        & 44.4 & 60.4 & 44.0 & 39.5
        & \cellcolor{secondcolor}\underline{40.4}
        & 15.4 & 64.9 & 44.1
        & 75.9 & 64.0
        \\
    RLSD
        & 87.9 & 75.0 & 90.9
        & \cellcolor{secondcolor}\underline{75.0}
        & \cellcolor{secondcolor}\underline{73.1}
        & 68.4 & 79.7
        & 41.5 & 58.6 & 42.3
        & \cellcolor{secondcolor}\underline{40.4}
        & 40.2
        & \cellcolor{secondcolor}\underline{16.8}
        & \cellcolor{secondcolor}\underline{66.9}
        & 43.8
        & 84.4 & 66.4
        \\
    SDAR
        & \cellcolor{secondcolor}\underline{97.1}
        & 62.5
        & \cellcolor{topcolor}\textbf{100.0}
        & 61.9
        & \cellcolor{topcolor}\textbf{75.0}
        & \cellcolor{topcolor}\textbf{84.2}
        & \cellcolor{secondcolor}\underline{84.4}
        & \cellcolor{secondcolor}\underline{44.8}
        & 58.1 & 44.3 & 38.6 & 36.2 & 15.7 & 66.1 & 43.4
        & \cellcolor{secondcolor}\underline{85.0}
        & \cellcolor{secondcolor}\underline{68.0}
        \\
    \textbf{\textsc{Seed} (Ours)}
        & \cellcolor{topcolor}\textbf{100.0}
        & \cellcolor{topcolor}\textbf{100.0}
        & \cellcolor{topcolor}\textbf{100.0}
        & \cellcolor{topcolor}\textbf{100.0}
        & 70.6
        & \cellcolor{secondcolor}\underline{80.0}
        & \cellcolor{topcolor}\textbf{91.8}
        & 44.3
        & \cellcolor{topcolor}\textbf{61.9}
        & \cellcolor{topcolor}\textbf{47.2}
        & \cellcolor{topcolor}\textbf{41.0}
        & \cellcolor{topcolor}\textbf{42.0}
        & \cellcolor{topcolor}\textbf{16.9}
        & \cellcolor{topcolor}\textbf{67.7}
        & \cellcolor{topcolor}\textbf{45.7}
        & \cellcolor{topcolor}\textbf{88.5}
        & \cellcolor{topcolor}\textbf{78.9}
        \\

    \midrule

    \rowcolor{gray!10}
    \multicolumn{18}{c}{\textit{Qwen2.5-7B-Instruct}} \\

    Vanilla
        & 36.1 & 22.2 & 3.1 & 0.0 & 0.0 & 0.0 & 12.5
        & 25.2 & 50.8 & 29.5 & 29.0 & 29.0 & 10.4 & 63.7 & 33.9
        & 5.9 & 1.6
        \\
    Skill-Prompt*
        & 51.7 & 50.0 & 32.3 & 5.3 & 4.3 & 0.0 & 23.4
        & 30.9 & 52.1 & 32.7 & 32.7 & 27.9 & 12.7 & 66.1 & 36.4
        & 1.7 & 0.8
        \\
    OPSD
        & 50.0 & 60.0 & 22.7 & 21.4 & 17.6 & 9.5 & 32.8
        & 8.8 & 8.6 & 17.5 & 2.5 & 4.2 & 0.5 & 1.2 & 6.2
        & 4.5 & 2.3
        \\
    GRPO
        & 91.2 & 87.5 & 96.2 & 81.0 & 65.0 & 57.9 & 81.2
        & 45.1 & 63.7 & 44.0 & 43.6 & 43.2 & 16.8 & 37.6 & 42.0
        & 80.9 & 72.6
        \\
    Skill-GRPO
        & 88.5 & 66.7 & 65.2 & 61.1 & 57.7 & 73.1 & 69.5
        & 45.2 & 63.7 & 45.7 & 43.1 & 43.3 & 19.6 & 21.4 & 40.3
        & 80.4 & 71.9
        \\
    Skill-GRPO*
        & \cellcolor{topcolor}\textbf{100.0}
        & 83.3
        & \cellcolor{secondcolor}\underline{96.4}
        & 83.3 & 75.0
        & \cellcolor{secondcolor}\underline{78.9}
        & \cellcolor{secondcolor}\underline{88.3}
        & 44.8 & 63.0 & 45.1 & 43.7 & 43.7
        & \cellcolor{secondcolor}\underline{20.5}
        & \cellcolor{secondcolor}\underline{71.4}
        & 47.5
        & 87.0
        & \cellcolor{secondcolor}\underline{81.2}
        \\
    GRPO+OPSD
        & 91.4 & 61.5
        & \cellcolor{topcolor}\textbf{100.0}
        & \cellcolor{secondcolor}\underline{87.5}
        & 76.5 & 52.2 & 80.4
        & \cellcolor{topcolor}\textbf{47.3}
        & \cellcolor{secondcolor}\underline{64.5}
        & 46.9 & 43.8 & 39.3 & 18.0 & 69.4 & 47.0
        & 86.8 & 76.5
        \\
    Skill-SD
        & 93.9
        & \cellcolor{secondcolor}\underline{93.8}
        & 90.9
        & \cellcolor{topcolor}\textbf{100.0}
        & 69.2 & 68.4 & 85.1
        & \cellcolor{secondcolor}\underline{47.1}
        & \cellcolor{secondcolor}\underline{64.5}
        & \cellcolor{secondcolor}\underline{47.8}
        & 44.2 & 42.1 & 20.2 & 69.0 & 47.8
        & 86.1 & 76.5
        \\
    RLSD
        & \cellcolor{topcolor}\textbf{100.0}
        & 87.5 & 92.3 & 58.8
        & \cellcolor{secondcolor}\underline{80.0}
        & 65.2 & 82.0
        & 46.8 & 63.0 & 44.4
        & \cellcolor{topcolor}\textbf{45.5}
        & \cellcolor{topcolor}\textbf{48.9}
        & \cellcolor{topcolor}\textbf{21.5}
        & \cellcolor{topcolor}\textbf{73.0}
        & \cellcolor{topcolor}\textbf{49.0}
        & 87.4 & 77.3
        \\
    SDAR
        & \cellcolor{secondcolor}\underline{94.7}
        & 75.0
        & \cellcolor{topcolor}\textbf{100.0}
        & 86.7 & 68.2
        & \cellcolor{secondcolor}\underline{78.9}
        & 85.9
        & 46.3 & 63.5
        & \cellcolor{topcolor}\textbf{48.2}
        & 43.8
        & \cellcolor{secondcolor}\underline{48.4}
        & 19.6
        & \cellcolor{topcolor}\textbf{73.0}
        & \cellcolor{topcolor}\textbf{49.0}
        & \cellcolor{secondcolor}\underline{89.4}
        & \cellcolor{topcolor}\textbf{82.8}
        \\
    \textbf{\textsc{Seed} (Ours)}
        & \cellcolor{topcolor}\textbf{100.0}
        & \cellcolor{topcolor}\textbf{100.0}
        & 96.3 & 80.0
        & \cellcolor{topcolor}\textbf{100.0}
        & \cellcolor{topcolor}\textbf{100.0}
        & \cellcolor{topcolor}\textbf{96.1}
        & 47.0
        & \cellcolor{topcolor}\textbf{64.9}
        & \cellcolor{secondcolor}\underline{47.8}
        & \cellcolor{secondcolor}\underline{45.2}
        & 45.3 & 20.1 & 70.2
        & \cellcolor{secondcolor}\underline{48.6}
        & \cellcolor{topcolor}\textbf{89.7}
        & 78.1
        \\

    \midrule

    \rowcolor{gray!10}
    \multicolumn{18}{c}{\textit{Qwen3-1.7B-Instruct}} \\

    Vanilla
        & 25.0 & 22.2 & 3.1 & 0.0 & 21.4 & 4.2 & 12.5
        & 29.4 & 46.9 & 37.0 & 23.5 & 19.6 & 6.4 & 10.5 & 24.8
        & 46.5 & 4.7
        \\
    Skill-Prompt*
        & 10.3 & 50.0 & 16.1 & 0.0 & 0.0 & 5.0 & 9.4
        & 29.4 & 46.5 & 36.2 & 22.9 & 20.8 & 4.3 & 10.1 & 24.3
        & 23.0 & 2.3
        \\
    OPSD
        & 26.3 & 33.3 & 9.1 & 0.0 & 4.5 & 5.3 & 14.1
        & 4.2 & 8.3 & 4.6 & 6.6 & 15.3 & 0.7 & 1.2 & 5.8
        & 47.4 & 9.3
        \\
    GRPO
        & 71.1 & 41.7 & 36.4 & 40.0 & 31.8 & 31.6 & 46.1
        & 40.0
        & \cellcolor{topcolor}\textbf{58.9}
        & 43.5 & 35.4 & 30.3 & 12.0 & 65.7 & 40.8
        & 67.3 & 38.3
        \\
    Skill-GRPO
        & 27.6
        & \cellcolor{secondcolor}\underline{54.5}
        & 22.7 & 27.3 & 0.0 & 19.2 & 21.1
        & 39.2
        & \cellcolor{secondcolor}\underline{58.6}
        & 43.9 & 35.2 & 28.2 & 11.5
        & \cellcolor{secondcolor}\underline{66.1}
        & 40.4
        & 73.4 & 46.1
        \\
    Skill-GRPO*
        & 31.4 & 42.9 & 51.9 & 8.3 & 11.5 & 7.1 & 28.1
        & 38.0 & 58.4 & 43.9 & 36.3 & 29.0 & 12.5
        & \cellcolor{topcolor}\textbf{66.9}
        & 40.7
        & 80.4 & 50.0
        \\
    GRPO+OPSD
        & 38.2 & 50.0 & 30.8 & 28.6 & 30.0 & 21.1 & 32.0
        & \cellcolor{secondcolor}\underline{40.7}
        & \cellcolor{topcolor}\textbf{58.9}
        & 45.0
        & \cellcolor{topcolor}\textbf{37.0}
        & 34.6
        & \cellcolor{topcolor}\textbf{13.3}
        & 65.7
        & \cellcolor{topcolor}\textbf{42.2}
        & 70.7 & 38.3
        \\
    Skill-SD
        & 52.9 & 37.5 & 69.2
        & \cellcolor{secondcolor}\underline{42.9}
        & \cellcolor{secondcolor}\underline{60.0}
        & \cellcolor{secondcolor}\underline{36.8}
        & 52.3
        & 39.1 & 57.5
        & \cellcolor{secondcolor}\underline{45.4}
        & 34.8 & 34.1 & 10.7 & 64.1 & 40.8
        & \cellcolor{secondcolor}\underline{81.8}
        & 53.9
        \\
    RLSD
        & 50.0 & 37.5 & 61.5 & 19.0 & 50.0 & 21.1 & 42.2
        & 38.6 & 57.3 & 43.0 & 34.5 & 34.1 & 11.5 & 65.3 & 40.6
        & 74.0 & 50.8
        \\
    SDAR
        & \cellcolor{secondcolor}\underline{73.5}
        & 25.0
        & \cellcolor{secondcolor}\underline{76.9}
        & 33.3 & 40.0
        & \cellcolor{secondcolor}\underline{36.8}
        & \cellcolor{secondcolor}\underline{53.9}
        & 39.7
        & \cellcolor{topcolor}\textbf{58.9}
        & 45.3 & 35.9
        & \cellcolor{topcolor}\textbf{35.5}
        & \cellcolor{secondcolor}\underline{12.6}
        & 65.3
        & \cellcolor{secondcolor}\underline{41.9}
        & 76.8
        & \cellcolor{secondcolor}\underline{58.6}
        \\
    \textbf{\textsc{Seed} (Ours)}
        & \cellcolor{topcolor}\textbf{97.6}
        & \cellcolor{topcolor}\textbf{100.0}
        & \cellcolor{topcolor}\textbf{100.0}
        & \cellcolor{topcolor}\textbf{80.0}
        & \cellcolor{topcolor}\textbf{84.2}
        & \cellcolor{topcolor}\textbf{90.0}
        & \cellcolor{topcolor}\textbf{92.0}
        & \cellcolor{topcolor}\textbf{42.0}
        & \cellcolor{topcolor}\textbf{58.9}
        & \cellcolor{topcolor}\textbf{47.1}
        & \cellcolor{secondcolor}\underline{36.9}
        & \cellcolor{secondcolor}\underline{35.2}
        & 10.6 & 64.9
        & \cellcolor{topcolor}\textbf{42.2}
        & \cellcolor{topcolor}\textbf{87.1}
        & \cellcolor{topcolor}\textbf{77.3}
        \\

    \bottomrule
    \end{tabular}
    }
\end{table*}

\paragraph{Evaluation Metrics.}
For ALFWorld, we report the success rate for each task family and their unweighted macro-average. For WebShop, we follow the environment protocol and report both the mean normalized task-completion score and the exact success rate. For Search-based QA, we compute answer accuracy separately on all seven subsets and report their unweighted macro-average. All metrics are expressed as percentages, and higher values indicate better performance.
\vspace{-0.08in}

\paragraph{Implementation Details.}
We use Qwen2.5-3B-Instruct and Qwen2.5-7B-Instruct~\citep{yang2024qwen25}, as well as Qwen3-1.7B-Instruct~\citep{yang2025qwen3}, as our backbone models.
\textit{SFT stage:}
For each backbone, we sample $M=180$ training tasks and collect $K_0=8$ rollout trajectories per task, resulting in 1,440 completed trajectories.
We then query GLM-5.2~\citep{zai2026glm52} as an external trajectory analyzer to extract hindsight skills from these trajectories.
After lightweight format validation, the retained trajectory--skill pairs are used to fine-tune the corresponding backbone for three epochs.
\textit{RL stage:}
The resulting SFT checkpoint initializes both the policy and the synchronized trajectory analyzer.
We train each model for 150 policy updates, using a batch size of 16 on ALFWorld and WebShop and 128 on Search-based QA.
The rollout group size is set to $N=8$ for all benchmarks.
Additional details on SFT data construction, skill annotation, optimization, and hyperparameter settings are provided in Appendix~\ref{app:implementation_details}.

\subsection{Main Results}\label{subsec32_main_results}
Table~\ref{tab:main_results} reports the results across different model scales and agentic domains. Three findings emerge:

\textbf{\textsc{Seed} consistently outperforms outcome-only RL through dense supervision.} Relative to GRPO, \textsc{Seed} improves the ALFWorld macro-average by 14.9-45.9 points, Search-based QA by 1.4-9.3 points, the WebShop task-completion score by 8.7-19.8 points, and the success rate by 5.5-39.0 points across the three backbones. Compared with Skill-GRPO, which conditions exploration on natural-language skills but still broadcasts a single terminal-reward-derived advantage to all valid tokens, \textsc{Seed} further improves the ALFWorld average by 26.6-70.9 points, Search-based QA by 1.8-11.6 points, the WebShop score by 9.3-13.7 points, and the success rate by 6.2-31.2 points. These consistent improvements over both GRPO and Skill-GRPO demonstrate that dense token-level hindsight supervision provides more effective credit assignment than outcome-only optimization, leading to substantially stronger performance across long-horizon agentic tasks.

\begin{figure}[t]
    \centering
    \includegraphics[width=\linewidth]{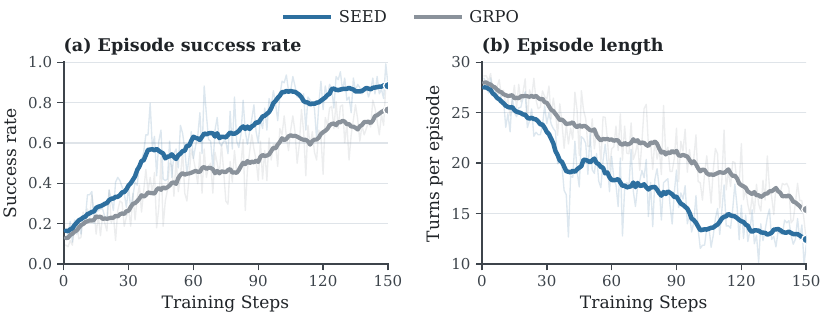}
    \caption{
        \textbf{Training dynamics on ALFWorld.}
        We compare \textsc{Seed} and GRPO using Qwen2.5-3B-Instruct as the backbone. Translucent curves show raw measurements, while solid curves show 13-point centered moving averages.
    }
    \label{fig:training_dynamics}
\end{figure}

\textbf{Internalizing skills is substantially more effective than inserted prompts.}
Skill-Prompt, which provides skills only during evaluation, underperforms \textsc{Seed} on every aggregate metric across all three backbones. \textsc{Seed} also exceeds Skill-GRPO* in 11 of the 12 aggregate comparisons, despite using no skill context during evaluation. These results indicate that hindsight skills are more effective when distilled into the policy than when supplied as additional context at inference time.

\textbf{Self-evolving hindsight distillation is stronger than static self-distillation.}
Across the static distillation baselines, \textsc{Seed} achieves the best or tied best result in 10 of the 12 aggregate comparisons. The advantage is clearest on ALFWorld, where \textsc{Seed} outperforms the strongest static baseline by 7.4 points with Qwen2.5-3B, 10.2 points with Qwen2.5-7B, and 38.1 points with Qwen3-1.7B. This pattern supports the benefit of synchronizing the analyzer with the latest policy, so hindsight supervision adapts to the trajectories and failure modes encountered during training.

\subsection{Training Dynamics}
Figure~\ref{fig:training_dynamics} compares the optimization dynamics of \textsc{Seed} and GRPO on ALFWorld. The success-rate curves diverge early: by training step 40, \textsc{Seed} reaches roughly 57\% while GRPO remains near 35\%, and the advantage persists thereafter. \textsc{Seed} also reduces the mean episode length more quickly, from approximately 28 turns to 13, compared with about 16 turns for GRPO at the end of training. Because shorter trajectories coincide with higher success, this reduction reflects more efficient task execution rather than premature termination. Together, these trends indicate that \textsc{Seed} improves task completion and interaction efficiency by reducing unproductive exploration and learning more direct solution strategies.

\subsection{Sample Efficiency}\label{subsec33_efficiency_analysis}
Figure~\ref{fig:sample_efficiency_line} shows that \textsc{Seed} consistently outperforms GRPO across all data fractions and can match or surpass the performance of GRPO trained with substantially more data. Using only 60\% of the training instances, \textsc{Seed} achieves a score of 80.7, exceeding the 75.0 obtained by GRPO with the full training set. Similarly, with 40\% of the data, \textsc{Seed} reaches 58.9, closely matching GRPO trained with twice as much data, which achieves 58.6 at the 80\% setting. These results demonstrate that \textsc{Seed} extracts more informative and effective supervision from each completed trajectory than outcome-only RL based solely on terminal rewards. Detailed results are provided in Appendix~\ref{app:sample_efficiency}.

\subsection{Cross-Domain Generalization}\label{subsec34_cross_domain_generalization}
We further evaluate the 3B checkpoint trained with \textsc{Seed} and GRPO on the ALFWorld unseen split. As shown in Figure~\ref{fig:generalization_bar}, \textsc{Seed} increases the macro-average success rate from 70.9 to 86.2, outperforming GRPO by 15.3 points and achieving better performance in five of the six task families. The largest improvements are observed on \textit{Heat} (+35.0), \textit{Look} (+18.3), and \textit{Pick} (+16.5). These results indicate that the policy trained with \textsc{Seed} acquires reusable behavioral strategies that transfer effectively to unseen environments, rather than merely memorizing the training trajectories. Detailed results are provided in Appendix~\ref{app:ood_generalization}.

\begin{figure}[ht!]
\centering
\begin{minipage}[t]{0.48\linewidth}
    \centering
    \includegraphics[width=\linewidth]{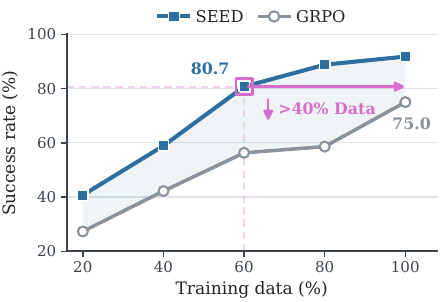}
    \vspace{-0.15in}
    \caption{
        \textbf{Sample efficiency analysis.} \textsc{Seed} consistently outperforms GRPO across different data fraction settings and surpasses full-data GRPO using only 60\% of the training data.
    }
    \label{fig:sample_efficiency_line}
\end{minipage}
\hfill
\begin{minipage}[t]{0.48\linewidth}
    \centering
    \includegraphics[width=\linewidth]{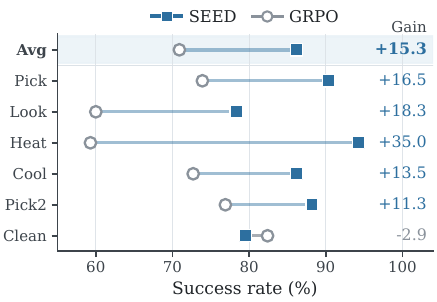}
    \vspace{-0.15in}
    \caption{
        \textbf{Cross-domain generalizability on ALFWorld Unseen.} \textsc{Seed} generally outperforms GRPO across unseen task types, demonstrating stronger cross-domain generalizability.
    }
    \label{fig:generalization_bar}
\end{minipage}
\end{figure}

\begin{table}[ht!]
\centering
\caption{
    \textbf{Ablation Results.}
    We report \textsc{Seed} performance and its ablated variants on ALFWorld.
}
\vspace{0.05in}
\label{tab:ablation_seed}
\small
\renewcommand{\arraystretch}{1.08}
\begin{tabular}{lccccccc}
\toprule
& \multicolumn{7}{c}{\textbf{ALFWorld}} \\
\cmidrule(lr){2-8}
\textbf{Method}
& \textbf{Pick}
& \textbf{Look}
& \textbf{Clean}
& \textbf{Heat}
& \textbf{Cool}
& \textbf{Pick2}
& \textbf{Avg.} \\
\midrule
\rowcolor{gray!10}
\textsc{Seed}
    & \textbf{100.0}
    & \textbf{100.0}
    & \textbf{100.0}
    & \textbf{100.0}
    & 70.6
    & 80.0
    & \textbf{91.8}
    \\
\;\;\;w/o Hindsight Skill SFT
    & 97.6
    & 81.8
    & 92.6
    & 90.0
    & \textbf{84.2}
    & 70.0
    & 86.0
    \\
\;\;\;w/o Self-Evolving OPD
    & 93.3
    & 70.0
    & 96.9
    & 84.6
    & 76.9
    & \textbf{100.0}
    & 87.0
    \\
\;\;\;w/o On-Policy Skill
    & 97.1
    & 62.5
    & \textbf{100.0}
    & 61.9
    & 75.0
    & 84.2
    & 84.4
    \\
\bottomrule
\end{tabular}
\end{table}

\subsection{Ablation Studies and Analysis}
Table~\ref{tab:ablation_seed} evaluates the contribution of the three core components of \textsc{Seed}.

\textbf{The impact of hindsight-skill SFT.}
Removing hindsight-skill SFT decreases the ALFWorld average from 91.8 to 86.0, corresponding to a 5.8-point drop. This result shows that equipping the actor and analyzer with an initial trajectory-analysis capability provides an important foundation for the subsequent self-evolving training process.

\textbf{The impact of self-evolving OPD.}
Removing self-evolving OPD reduces the average performance to 87.0, a decrease of 4.8 points. This indicates that the one-time skill supervision introduced during SFT is insufficient on its own, and that continuously distilling supervision from newly collected trajectories is necessary to keep pace with the evolving policy.

\textbf{The impact of on-policy skills.}
Replacing on-policy skills with skills from a static offline library leads to the largest performance degradation, lowering the average to 84.4 by 7.4 points. This finding highlights the importance of deriving guidance from the current policy's own trajectories, as such skills better reflect its evolving behaviors and failure modes.

\paragraph{Qualitative Analysis. }
Figure~\ref{fig:case_comparison} presents a representative example of the behavioral differences between GRPO and \textsc{Seed}. GRPO visits the target receptacle before locating the required object, interacts with irrelevant objects, and repeatedly explores unproductive locations until reaching the step limit. In contrast, \textsc{Seed} follows a more coherent and goal-directed trajectory: it searches plausible storage locations, identifies the candle on the second shelf, and completes the task with the correct placement action. The more direct trajectory exhibited by \textsc{Seed} is consistent with its stronger performance and higher interaction efficiency reported in the preceding experiments.

\begin{figure}[ht!]
    \centering
    \includegraphics[width=\linewidth]{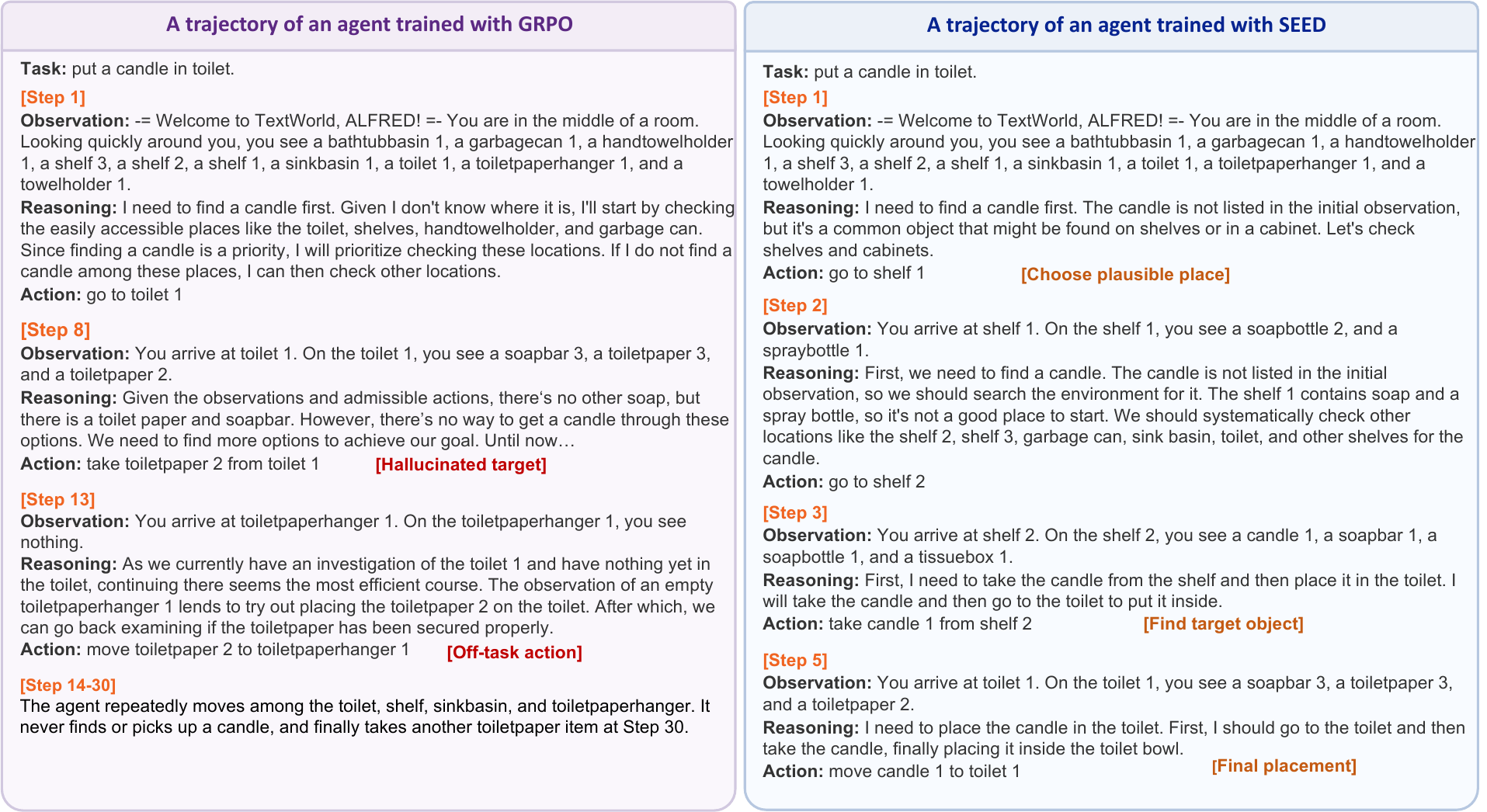}
    \caption{
        \textbf{Qualitative comparison on ALFWorld.}
        For the task “put a candle in toilet,” the GRPO-
        trained agent searches the target receptacle before locating the candle, takes an irrelevant toilet-paper
        item, and then enters an off-task loop. In contrast, SEED systematically checks plausible shelves,
        finds the candle, and completes the required placement in five steps.
    }
    \label{fig:case_comparison}
\end{figure}

\section{Conclusion}
We introduced \textsc{Seed}, a self-evolving on-policy distillation framework for long-horizon agentic reinforcement learning. \textsc{Seed} bridges sparse trajectory-level outcomes and token-level policy learning by extracting hindsight skills from completed on-policy trajectories and distilling their behavioral effects into dense supervision. The latest policy serves as both the actor and the analyzer, allowing its behavior and the supervision derived from its experience to evolve together while remaining aligned with the current trajectory distribution. Joint optimization with outcome-based RL enables the policy to internalize reusable guidance without relying on skills at inference time. Experiments across embodied interaction, web navigation, search-based QA, visual reasoning and planning demonstrate consistent performance improvements, sample efficiency, and robustness over powerful baselines.

\bibliography{iclr2025_conference}

@article{xi2025rise,
  title   = {The Rise and Potential of Large Language Model Based Agents: A Survey},
  author  = {Xi, Zhiheng and Chen, Wenxiang and Guo, Xin and He, Wei and Ding, Yiwen and Hong, Boyang and Zhang, Ming and Wang, Junzhe and Jin, Senjie and Zhou, Enyu and Zheng, Rui and Fan, Xiaoran and Wang, Xiao and Xiong, Limao and Zhou, Yuhao and Wang, Weiran and Jiang, Changhao and Zou, Yicheng and Liu, Xiangyang and Yin, Zhangyue and Dou, Shihan and Weng, Rongxiang and Cheng, Wensen and Zhang, Qi and Qin, Wenjuan and Zheng, Yongyan and Qiu, Xipeng and Huang, Xuanjing and Gui, Tao},
  journal = {Science China Information Sciences},
  year    = {2025},
  doi     = {10.1007/s11432-024-4222-0}
}

@inproceedings{schick2023toolformer,
  title     = {Toolformer: Language Models Can Teach Themselves to Use Tools},
  author    = {Schick, Timo and Dwivedi-Yu, Jane and Dess{\`i}, Roberto and Raileanu, Roberta and Lomeli, Maria and Zettlemoyer, Luke and Cancedda, Nicola and Scialom, Thomas},
  booktitle = {Advances in Neural Information Processing Systems},
  year      = {2023},
  url       = {https://arxiv.org/abs/2302.04761}
}

@inproceedings{patil2024gorilla,
  title     = {Gorilla: Large Language Model Connected with Massive {APIs}},
  author    = {Patil, Shishir G. and Zhang, Tianjun and Wang, Xin and Gonzalez, Joseph E.},
  booktitle = {Advances in Neural Information Processing Systems},
  year      = {2024},
  url       = {https://arxiv.org/abs/2305.15334}
}

@misc{liu2023agentbench,
  title         = {AgentBench: Evaluating {LLMs} as Agents},
  author        = {Liu, Xiao and Yu, Hao and Zhang, Hanchen and Xu, Yifan and Lei, Xuanyu and Lai, Hanyu and Gu, Yu and Ding, Hangliang and Men, Kaiwen and Yang, Kejuan and Zhang, Shudan and Deng, Xiang and Zeng, Aohan and Du, Zhengxiao and Zhang, Chenhui and Shen, Sheng and Zhang, Tianjun and Su, Yu and Sun, Huan and Huang, Minlie and Dong, Yuxiao and Tang, Jie},
  year          = {2023},
  eprint        = {2308.03688},
  archivePrefix = {arXiv},
  primaryClass  = {cs.CL},
  url           = {https://arxiv.org/abs/2308.03688}
}

@inproceedings{shridhar2021alfworld,
  title     = {{ALFWorld}: Aligning Text and Embodied Environments for Interactive Learning},
  author    = {Shridhar, Mohit and Yuan, Xingdi and C{\^o}t{\'e}, Marc-Alexandre and Bisk, Yonatan and Trischler, Adam and Hausknecht, Matthew},
  booktitle = {International Conference on Learning Representations},
  year      = {2021},
  url       = {https://arxiv.org/abs/2010.03768}
}

@inproceedings{yao2022webshop,
  title     = {{WebShop}: Towards Scalable Real-World Web Interaction with Grounded Language Agents},
  author    = {Yao, Shunyu and Chen, Howard and Yang, John and Narasimhan, Karthik},
  booktitle = {Advances in Neural Information Processing Systems},
  year      = {2022},
  url       = {https://arxiv.org/abs/2207.01206}
}

@misc{jin2025searchr1,
  title         = {{Search-R1}: Training {LLMs} to Reason and Leverage Search Engines with Reinforcement Learning},
  author        = {Jin, Bowen and Zeng, Hansi and Yue, Zhenrui and Wang, Dong and Zamani, Hamed and Han, Jiawei},
  year          = {2025},
  eprint        = {2503.09516},
  archivePrefix = {arXiv},
  primaryClass  = {cs.CL},
  url           = {https://arxiv.org/abs/2503.09516}
}

@article{kwiatkowski2019natural,
  title     = {Natural Questions: A Benchmark for Question Answering Research},
  author    = {Kwiatkowski, Tom and Palomaki, Jennimaria and Redfield, Olivia and Collins, Michael and Parikh, Ankur and Alberti, Chris and Epstein, Danielle and Polosukhin, Illia and Devlin, Jacob and Lee, Kenton and Toutanova, Kristina and Jones, Llion and Kelcey, Matthew and Chang, Ming-Wei and Dai, Andrew M. and Uszkoreit, Jakob and Le, Quoc and Petrov, Slav},
  journal   = {Transactions of the Association for Computational Linguistics},
  volume    = {7},
  pages     = {452--466},
  year      = {2019},
  doi       = {10.1162/tacl_a_00276},
  url       = {https://aclanthology.org/Q19-1026/}
}

@inproceedings{joshi2017triviaqa,
  title     = {{TriviaQA}: A Large Scale Distantly Supervised Challenge Dataset for Reading Comprehension},
  author    = {Joshi, Mandar and Choi, Eunsol and Weld, Daniel and Zettlemoyer, Luke},
  booktitle = {Proceedings of the 55th Annual Meeting of the Association for Computational Linguistics},
  pages     = {1601--1611},
  year      = {2017},
  doi       = {10.18653/v1/P17-1147},
  url       = {https://aclanthology.org/P17-1147/}
}

@inproceedings{mallen2023popqa,
  title     = {When Not to Trust Language Models: Investigating Effectiveness of Parametric and Non-Parametric Memories},
  author    = {Mallen, Alex and Asai, Akari and Zhong, Victor and Das, Rajarshi and Khashabi, Daniel and Hajishirzi, Hannaneh},
  booktitle = {Proceedings of the 61st Annual Meeting of the Association for Computational Linguistics},
  pages     = {9802--9822},
  year      = {2023},
  doi       = {10.18653/v1/2023.acl-long.546},
  url       = {https://aclanthology.org/2023.acl-long.546/}
}

@inproceedings{yang2018hotpotqa,
  title     = {{HotpotQA}: A Dataset for Diverse, Explainable Multi-hop Question Answering},
  author    = {Yang, Zhilin and Qi, Peng and Zhang, Saizheng and Bengio, Yoshua and Cohen, William and Salakhutdinov, Ruslan and Manning, Christopher D.},
  booktitle = {Proceedings of the 2018 Conference on Empirical Methods in Natural Language Processing},
  pages     = {2369--2380},
  year      = {2018},
  doi       = {10.18653/v1/D18-1259},
  url       = {https://aclanthology.org/D18-1259/}
}

@inproceedings{ho2020constructing,
  title     = {Constructing A Multi-hop {QA} Dataset for Comprehensive Evaluation of Reasoning Steps},
  author    = {Ho, Xanh and Duong Nguyen, Anh-Khoa and Sugawara, Saku and Aizawa, Akiko},
  booktitle = {Proceedings of the 28th International Conference on Computational Linguistics},
  pages     = {6609--6625},
  year      = {2020},
  doi       = {10.18653/v1/2020.coling-main.580},
  url       = {https://aclanthology.org/2020.coling-main.580/}
}

@article{trivedi2022musique,
  title     = {{MuSiQue}: Multihop Questions via Single-hop Question Composition},
  author    = {Trivedi, Harsh and Balasubramanian, Niranjan and Khot, Tushar and Sabharwal, Ashish},
  journal   = {Transactions of the Association for Computational Linguistics},
  volume    = {10},
  pages     = {539--554},
  year      = {2022},
  doi       = {10.1162/tacl_a_00475},
  url       = {https://aclanthology.org/2022.tacl-1.31/}
}

@inproceedings{press2023measuring,
  title     = {Measuring and Narrowing the Compositionality Gap in Language Models},
  author    = {Press, Ofir and Zhang, Muru and Min, Sewon and Schmidt, Ludwig and Smith, Noah and Lewis, Mike},
  booktitle = {Findings of the Association for Computational Linguistics: EMNLP 2023},
  pages     = {5687--5711},
  year      = {2023},
  doi       = {10.18653/v1/2023.findings-emnlp.378},
  url       = {https://aclanthology.org/2023.findings-emnlp.378/}
}

@misc{shao2024deepseekmath,
  title         = {{DeepSeekMath}: Pushing the Limits of Mathematical Reasoning in Open Language Models},
  author        = {Shao, Zhihong and Wang, Peiyi and Zhu, Qihao and Xu, Runxin and Song, Junxiao and Bi, Xiao and Zhang, Haowei and Zhang, Mingchuan and Li, Y. K. and Wu, Y. and Guo, Daya},
  year          = {2024},
  eprint        = {2402.03300},
  archivePrefix = {arXiv},
  primaryClass  = {cs.CL},
  url           = {https://arxiv.org/abs/2402.03300}
}

@misc{wang2025ragen,
  title         = {{RAGEN}: Understanding Self-Evolution in {LLM} Agents via Multi-Turn Reinforcement Learning},
  author        = {Wang, Zihan and Wang, Kangrui and Wang, Qineng and Zhang, Pingyue and Li, Linjie and Yang, Zhengyuan and Jin, Xing and Yu, Kefan and Nguyen, Minh Nhat and Liu, Licheng and Gottlieb, Eli and Lu, Yiping and Cho, Kyunghyun and Wu, Jiajun and Fei-Fei, Li and Wang, Lijuan and Choi, Yejin and Li, Manling},
  year          = {2025},
  eprint        = {2504.20073},
  archivePrefix = {arXiv},
  primaryClass  = {cs.LG},
  url           = {https://arxiv.org/abs/2504.20073}
}

@misc{luo2025agentlightning,
  title         = {Agent Lightning: Train {ANY} {AI} Agents with Reinforcement Learning},
  author        = {Luo, Xufang and Zhang, Yuge and He, Zhiyuan and Wang, Zilong and Zhao, Siyun and Li, Dongsheng and Qiu, Luna K. and Yang, Yuqing},
  year          = {2025},
  eprint        = {2508.03680},
  archivePrefix = {arXiv},
  primaryClass  = {cs.LG},
  url           = {https://arxiv.org/abs/2508.03680}
}

@inproceedings{arjona2019rudder,
  title     = {{RUDDER}: Return Decomposition for Delayed Rewards},
  author    = {Arjona-Medina, Jose A. and Gillhofer, Michael and Widrich, Michael and Unterthiner, Thomas and Brandstetter, Johannes and Hochreiter, Sepp},
  booktitle = {Advances in Neural Information Processing Systems},
  year      = {2019},
  url       = {https://arxiv.org/abs/1806.07857}
}

@article{
zhang2026the,
title={The Landscape of Agentic Reinforcement Learning for {LLM}s: A Survey},
author={Guibin Zhang and Hejia Geng and Xiaohang Yu and Zhenfei Yin and Zaibin Zhang and Zelin Tan and Heng Zhou and Zhong-Zhi Li and Xiangyuan Xue and Yijiang Li and Yifan Zhou and Yang Chen and Chen Zhang and Yutao Fan and Zihu Wang and Songtao Huang and Francisco Piedrahita Velez and Yue Liao and Hongru WANG and Mengyue Yang and Heng Ji and Jun Wang and Shuicheng YAN and Philip Torr and LEI BAI},
journal={Transactions on Machine Learning Research},
issn={2835-8856},
year={2026},
note={Survey Certification}
}

@misc{uesato2022solving,
  title         = {Solving Math Word Problems with Process- and Outcome-Based Feedback},
  author        = {Uesato, Jonathan and Kushman, Nate and Kumar, Ramana and Song, Francis and Siegel, Noah and Wang, Lisa and Creswell, Antonia and Irving, Geoffrey and Higgins, Irina},
  year          = {2022},
  eprint        = {2211.14275},
  archivePrefix = {arXiv},
  primaryClass  = {cs.LG},
  url           = {https://arxiv.org/abs/2211.14275}
}

@inproceedings{lightman2024lets,
  title     = {Let's Verify Step by Step},
  author    = {Lightman, Hunter and Kosaraju, Vineet and Burda, Yura and Edwards, Harri and Baker, Bowen and Lee, Teddy and Leike, Jan and Schulman, John and Sutskever, Ilya and Cobbe, Karl},
  booktitle = {International Conference on Learning Representations},
  year      = {2024},
  url       = {https://arxiv.org/abs/2305.20050}
}

@inproceedings{shinn2023reflexion,
  title     = {Reflexion: Language Agents with Verbal Reinforcement Learning},
  author    = {Shinn, Noah and Cassano, Federico and Berman, Edward and Gopinath, Ashwin and Narasimhan, Karthik and Yao, Shunyu},
  booktitle = {Advances in Neural Information Processing Systems},
  year      = {2023},
  url       = {https://arxiv.org/abs/2303.11366}
}

@inproceedings{zhao2024expel,
  title     = {{ExpeL}: {LLM} Agents Are Experiential Learners},
  author    = {Zhao, Andrew and Huang, Daniel and Xu, Quentin and Lin, Matthieu and Liu, Yong-Jin and Huang, Gao},
  booktitle = {Proceedings of the AAAI Conference on Artificial Intelligence},
  year      = {2024},
  url       = {https://arxiv.org/abs/2308.10144}
}

@misc{wang2023voyager,
  title         = {Voyager: An Open-Ended Embodied Agent with Large Language Models},
  author        = {Wang, Guanzhi and Xie, Yuqi and Jiang, Yunfan and Mandlekar, Ajay and Xiao, Chaowei and Zhu, Yuke and Fan, Linxi and Anandkumar, Anima},
  year          = {2023},
  eprint        = {2305.16291},
  archivePrefix = {arXiv},
  primaryClass  = {cs.AI},
  url           = {https://arxiv.org/abs/2305.16291}
}

@inproceedings{madaan2023selfrefine,
  title     = {Self-Refine: Iterative Refinement with Self-Feedback},
  author    = {Madaan, Aman and Tandon, Niket and Gupta, Prakhar and Hallinan, Skyler and Gao, Luyu and Wiegreffe, Sarah and Alon, Uri and Dziri, Nouha and Prabhumoye, Shrimai and Yang, Yiming and Gupta, Shashank and Majumder, Bodhisattwa Prasad and Hermann, Katherine and Welleck, Sean and Yazdanbakhsh, Amir and Clark, Peter},
  booktitle = {Advances in Neural Information Processing Systems},
  year      = {2023},
  url       = {https://arxiv.org/abs/2303.17651}
}

@misc{hinton2015distilling,
  title         = {Distilling the Knowledge in a Neural Network},
  author        = {Hinton, Geoffrey and Vinyals, Oriol and Dean, Jeff},
  year          = {2015},
  eprint        = {1503.02531},
  archivePrefix = {arXiv},
  primaryClass  = {stat.ML},
  url           = {https://arxiv.org/abs/1503.02531}
}

@inproceedings{kim2016sequence,
  title     = {Sequence-Level Knowledge Distillation},
  author    = {Kim, Yoon and Rush, Alexander M.},
  booktitle = {Proceedings of the 2016 Conference on Empirical Methods in Natural Language Processing},
  pages     = {1317--1327},
  year      = {2016},
  doi       = {10.18653/v1/D16-1139},
  url       = {https://aclanthology.org/D16-1139}
}

@inproceedings{ross2011reduction,
  title     = {A Reduction of Imitation Learning and Structured Prediction to No-Regret Online Learning},
  author    = {Ross, St{\'e}phane and Gordon, Geoffrey J. and Bagnell, J. Andrew},
  booktitle = {Proceedings of the Fourteenth International Conference on Artificial Intelligence and Statistics},
  pages     = {627--635},
  year      = {2011},
  url       = {https://proceedings.mlr.press/v15/ross11a.html}
}

@inproceedings{agarwal2024onpolicy,
  title     = {On-Policy Distillation of Language Models: Learning from Self-Generated Mistakes},
  author    = {Agarwal, Rishabh and Vieillard, Nino and Zhou, Yongchao and Stanczyk, Piotr and Ramos, Sabela and Geist, Matthieu and Bachem, Olivier},
  booktitle = {International Conference on Learning Representations},
  year      = {2024},
  url       = {https://arxiv.org/abs/2306.13649}
}

@misc{zhao2026selfdistilledreasoner,
  title         = {Self-Distilled Reasoner: On-Policy Self-Distillation for Large Language Models},
  author        = {Zhao, Siyan and Xie, Zhihui and Liu, Mengchen and Huang, Jing and Pang, Guan and Chen, Feiyu and Grover, Aditya},
  year          = {2026},
  eprint        = {2601.18734},
  archivePrefix = {arXiv},
  primaryClass  = {cs.LG},
  url           = {https://arxiv.org/abs/2601.18734}
}

@misc{hubotter2026sdpo,
  title         = {Reinforcement Learning via Self-Distillation},
  author        = {H{\"u}botter, Jonas and L{\"u}beck, Frederike and Behric, Lejs and Baumann, Anton and Bagatella, Marco and Marta, Daniel and Hakimi, Ido and Shenfeld, Idan and Kleine Buening, Thomas and Guestrin, Carlos and Krause, Andreas},
  year          = {2026},
  eprint        = {2601.20802},
  archivePrefix = {arXiv},
  primaryClass  = {cs.LG},
  url           = {https://arxiv.org/abs/2601.20802}
}

@misc{wang2026skillsd,
  title         = {{Skill-SD}: Skill-Conditioned Self-Distillation for Multi-turn {LLM} Agents},
  author        = {Wang, Hao and Wang, Guozhi and Xiao, Han and Zhou, Yufeng and Pan, Yue and Wang, Jichao and Xu, Ke and Wen, Yafei and Ruan, Xiaohu and Chen, Xiaoxin and Qi, Honggang},
  year          = {2026},
  eprint        = {2604.10674},
  archivePrefix = {arXiv},
  primaryClass  = {cs.LG},
  url           = {https://arxiv.org/abs/2604.10674}
}

@misc{yang2026rlsd,
  title         = {Self-Distilled {RLVR}},
  author        = {Yang, Chenxu and Qin, Chuanyu and Si, Qingyi and Chen, Minghui and Gu, Naibin and Yao, Dingyu and Lin, Zheng and Wang, Weiping and Wang, Jiaqi and Duan, Nan},
  year          = {2026},
  eprint        = {2604.03128},
  archivePrefix = {arXiv},
  url           = {https://arxiv.org/abs/2604.03128}
}

@misc{lu2026sdar,
  title         = {Self-Distilled Agentic Reinforcement Learning},
  author        = {Lu, Zhengxi and Yao, Zhiyuan and Han, Zhuowen and Wang, Zi-Han and Wu, Jinyang and Gu, Qi and Cai, Xunliang and Lu, Weiming and Xiao, Jun and Zhuang, Yueting and Shen, Yongliang},
  year          = {2026},
  eprint        = {2605.15155},
  archivePrefix = {arXiv},
  primaryClass  = {cs.LG},
  url           = {https://arxiv.org/abs/2605.15155}
}

@misc{zhong2026sod,
  title         = {{SOD}: Step-wise On-policy Distillation for Small Language Model Agents},
  author        = {Zhong, Qiyong and Zheng, Mao and Song, Mingyang and Lin, Xin and Sun, Jie and Jiang, Houcheng and Wang, Xiang and Fang, Junfeng},
  year          = {2026},
  eprint        = {2605.07725},
  archivePrefix = {arXiv},
  primaryClass  = {cs.CL},
  url           = {https://arxiv.org/abs/2605.07725}
}

@misc{ko2026reopold,
  title         = {Scaling Reasoning Efficiently via Relaxed On-Policy Distillation},
  author        = {Ko, Jongwoo and Abdali, Sara and Kim, Young Jin and Chen, Tianyi and Cameron, Pashmina},
  year          = {2026},
  eprint        = {2603.11137},
  archivePrefix = {arXiv},
  primaryClass  = {cs.LG},
  url           = {https://arxiv.org/abs/2603.11137}
}

@article{yang2026opid,
  title={OPID: On-Policy Skill Distillation for Agentic Reinforcement Learning},
  author={Yang, Shuo and Wu, Jinyang and Lu, Zhengxi and Shen, Yuhao and Zhang, Fan and Feng, Lang and Zhang, Shuai and Luo, Haoran and Lian, Zheng and Wen, Zhengqi and others},
  journal={arXiv preprint arXiv:2606.26790},
  year={2026}
}

@article{wu2024beyond,
  title={Beyond examples: High-level automated reasoning paradigm in in-context learning via mcts},
  author={Wu, Jinyang and Feng, Mingkuan and Zhang, Shuai and Che, Feihu and Wen, Zengqi and Liao, Chonghua and Tao, Jianhua},
  journal={arXiv preprint arXiv:2411.18478},
  year={2024}
}

@article{andrychowicz2017hindsight,
  title={Hindsight experience replay},
  author={Andrychowicz, Marcin and Wolski, Filip and Ray, Alex and Schneider, Jonas and Fong, Rachel and Welinder, Peter and McGrew, Bob and Tobin, Josh and Pieter Abbeel, OpenAI and Zaremba, Wojciech},
  journal={Advances in neural information processing systems},
  volume={30},
  year={2017}
}

@inproceedings{shen2026double,
    title = "Double: Breaking the Acceleration Limit via Double Retrieval Speculative Parallelism",
    author = "Shen, Yuhao  and
      Liu, Tianyu  and
      Shen, Junyi  and
      Wu, Jinyang  and
      Kong, Quan  and
      Li, Huan  and
      Wang, Cong",
    booktitle = "Proceedings of the 64th Annual Meeting of the {A}ssociation for {C}omputational {L}inguistics (Volume 1: Long Papers)",
    month = jul,
    year = "2026",
    address = "San Diego, California, United States",
    publisher = "Association for Computational Linguistics",
    pages = "19242--19263",
}

@article{luo2025large,
  title={Large language model agent: A survey on methodology, applications and challenges},
  author={Luo, Junyu and Zhang, Weizhi and Yuan, Ye and Zhao, Yusheng and Yang, Junwei and Gu, Yiyang and Wu, Bohan and Chen, Binqi and Qiao, Ziyue and Long, Qingqing and others},
  journal={arXiv preprint arXiv:2503.21460},
  year={2025}
}

@article{fang2026roboteq,
  title={RobotEQ: Transitioning from Passive Intelligence to Active Intelligence in Embodied AI},
  author={Fang, Kuofei and Che, Xinyi and Ouyang, Haomin and Zhang, Shufan and Wang, Xuehao and Liu, Qi and Liu, Liyi and Zhang, Chenqi and Cai, Wenxi and Dai, Wenyu and others},
  journal={arXiv preprint arXiv:2605.06234},
  year={2026}
}

@inproceedings{wu2026spark,
    title = "{SPARK}: Strategic Policy-Aware Exploration via Dynamic Branching for Long-Horizon Agentic Learning",
    author = "Wu, Jinyang  and
      Yang, Shuo  and
      Shen, Yuhao  and
      Zhang, Shuai  and
      Wen, Zhengqi  and
      Tao, Jianhua",
    booktitle = "Proceedings of the 64th Annual Meeting of the {A}ssociation for {C}omputational {L}inguistics (Volume 1: Long Papers)",
    month = jul,
    year = "2026",
    address = "San Diego, California, United States",
    publisher = "Association for Computational Linguistics",
    pages = "23981--24004",
}

@article{xu2026odysseyarena,
  title={OdysseyArena: Benchmarking Large Language Models For Long-Horizon, Active and Inductive Interactions},
  author={Xu, Fangzhi and Yan, Hang and Sun, Qiushi and Wu, Jinyang and Huang, Zixian and Huang, Muye and Gong, Jingyang and Ding, Zichen and Cheng, Kanzhi and Wang, Yian and others},
  journal={arXiv preprint arXiv:2602.05843},
  year={2026}
}

@article{cheng2026dspark,
  title={DSpark: Confidence-Scheduled Speculative Decoding with Semi-Autoregressive Generation},
  author={Cheng, Xin and Yu, Xingkai and Shao, Chenze and Li, Jiashi and Xiong, Yunfan and Qian, Yi and Zhu, Jiaqi and Ma, Shirong and Zhang, Xiaokang and Ye, Jiasheng and others},
  journal={arXiv preprint arXiv:2607.05147},
  year={2026}
}

@inproceedings{wu2026atlas,
    title = "{ATLAS}: Orchestrating Heterogeneous Models and Tools for Multi-Domain Complex Reasoning",
    author = "Wu, Jinyang  and
      Zhai, Guocheng  and
      Jin, Ruihan  and
      Yuan, Jiahao  and
      Shen, Yuhao  and
      Zhang, Shuai  and
      Wen, Zhengqi  and
      Tao, Jianhua",
    booktitle = "Findings of the {A}ssociation for {C}omputational {L}inguistics: {ACL} 2026",
    month = jul,
    year = "2026",
    address = "San Diego, California, United States",
    publisher = "Association for Computational Linguistics",
    pages = "17503--17535",
}

@misc{li2026longhorizonterminal,
      title={Long-Horizon-Terminal-Bench: Testing the Limits of Agents on Long-Horizon Terminal Tasks with Dense Reward-Based Grading}, 
      author={Zongxia Li and Zhongzhi Li and Yucheng Shi and Ruhan Wang and Junyao Yang and Zhichao Liu and Xiyang Wu and Anhao Li and Yue Yu and Ninghao Liu and Lichao Sun and Haotao Mi and LeoweiLiang},
      year={2026},
      eprint={2607.08964},
      archivePrefix={arXiv},
      primaryClass={cs.AI},
      url={https://arxiv.org/abs/2607.08964}, 
}

@article{zhang2026deepplanning,
  title={DeepPlanning: Benchmarking Long-Horizon Agentic Planning with Verifiable Constraints},
  author={Zhang, Yinger and Jiang, Shutong and Li, Renhao and Tu, Jianhong and Su, Yang and Deng, Lianghao and Guo, Xudong and Lv, Chenxu and Lin, Junyang},
  journal={arXiv preprint arXiv:2601.18137},
  year={2026}
}

@article{lu2026skill0,
  title={Skill0: In-context agentic reinforcement learning for skill internalization},
  author={Lu, Zhengxi and Yao, Zhiyuan and Wu, Jinyang and Han, Chengcheng and Gu, Qi and Cai, Xunliang and Lu, Weiming and Xiao, Jun and Zhuang, Yueting and Shen, Yongliang},
  journal={arXiv preprint arXiv:2604.02268},
  year={2026}
}

@article{yang2024qwen25,
  title   = {Qwen2.5 Technical Report},
  author  = {Yang, An and Yang, Baosong and Zhang, Beichen and Hui, Binyuan
             and Zheng, Bo and Yu, Bowen and Li, Chengyuan and Liu, Dayiheng
             and Huang, Fei and Wei, Haoran and Lin, Huan and Yang, Jian
             and Tu, Jianhong and Zhang, Jianwei and Yang, Jianxin
             and Yang, Jiaxi and Zhou, Jingren and Lin, Junyang and Dang, Kai
             and Lu, Keming and Bao, Keqin and Yang, Kexin and Yu, Le
             and Li, Mei and Xue, Mingfeng and Zhang, Pei and Zhu, Qin
             and Men, Rui and Lin, Runji and Li, Tianhao and Tang, Tianyi
             and Xia, Tingyu and Ren, Xingzhang and Ren, Xuancheng
             and Fan, Yang and Su, Yang and Zhang, Yichang and Wan, Yu
             and Liu, Yuqiong and Cui, Zeyu and Zhang, Zhenru and Qiu, Zihan},
  journal = {arXiv preprint arXiv:2412.15115},
  year    = {2024}
}

@article{yang2025qwen3,
  title   = {Qwen3 Technical Report},
  author  = {Yang, An and Li, Anfeng and Yang, Baosong and Zhang, Beichen
             and Hui, Binyuan and Zheng, Bo and Yu, Bowen and Gao, Chang
             and Huang, Chengen and Lv, Chenxu and Zheng, Chujie
             and Liu, Dayiheng and Zhou, Fan and Huang, Fei and Hu, Feng
             and Ge, Hao and Wei, Haoran and Lin, Huan and Tang, Jialong
             and Yang, Jian and Tu, Jianhong and Zhang, Jianwei
             and Yang, Jianxin and Yang, Jiaxi and Zhou, Jing
             and Zhou, Jingren and Lin, Junyang and Dang, Kai and Bao, Keqin
             and Yang, Kexin and Yu, Le and Deng, Lianghao and Li, Mei
             and Xue, Mingfeng and Li, Mingze and Zhang, Pei and Wang, Peng
             and Zhu, Qin and Men, Rui and Gao, Ruize and Liu, Shixuan
             and Luo, Shuang and Li, Tianhao and Tang, Tianyi and Yin, Wenbiao
             and Ren, Xingzhang and Wang, Xinyu and Zhang, Xinyu
             and Ren, Xuancheng and Fan, Yang and Su, Yang and Zhang, Yichang
             and Zhang, Yinger and Wan, Yu and Liu, Yuqiong and Wang, Zekun
             and Cui, Zeyu and Zhang, Zhenru and Zhou, Zhipeng and Qiu, Zihan},
  journal = {arXiv preprint arXiv:2505.09388},
  year    = {2025}
}

@article{bai2025qwen2,
  title={Qwen2. 5-vl technical report},
  author={Bai, Shuai and Chen, Keqin and Liu, Xuejing and Wang, Jialin and Ge, Wenbin and Song, Sibo and Dang, Kai and Wang, Peng and Wang, Shijie and Tang, Jun and others},
  journal={arXiv preprint arXiv:2502.13923},
  year={2025}
}

@inproceedings{NEURIPS2024_c848b7d3,
 author = {Zhai, Yuexiang and Bai, Hao and Lin, Zipeng and Pan, Jiayi and Tong, Shengbang and Zhou, Yifei and Suhr, Alane and Xie, Saining and LeCun, Yann and Ma, Yi and Levine, Sergey},
 booktitle = {Advances in Neural Information Processing Systems},
 doi = {10.52202/079017-3522},
 editor = {A. Globerson and L. Mackey and D. Belgrave and A. Fan and U. Paquet and J. Tomczak and C. Zhang},
 pages = {110935--110971},
 publisher = {Curran Associates, Inc.},
 title = {Fine-Tuning Large Vision-Language Models as Decision-Making Agents via Reinforcement Learning},
 url = {https://proceedings.neurips.cc/paper_files/paper/2024/file/c848b7d3adc08fcd0bf1df3101ba6728-Paper-Conference.pdf},
 volume = {37},
 year = {2024}
}

@misc{SchraderSokoban2018,
  author = {Schrader, Max-Philipp B.},
  title = {gym-sokoban},
  year = {2018},
  publisher = {GitHub},
  journal = {GitHub repository},
  howpublished = {\url{https://github.com/mpSchrader/gym-sokoban}},
  commit = {#CommitId}
}

@misc{zai2026glm52,
  title        = {{GLM-5.2: Built for Long-Horizon Tasks}},
  author       = {{Z.ai}},
  year         = {2026},
  month        = jun,
  day          = {16},
  howpublished = {\url{https://z.ai/blog/glm-5.2}},
  note         = {Accessed: 2026-06-22}
}

@inproceedings{jiang2025selfincorrect,
  title={SELF-[IN]CORRECT: LLMs Struggle with Discriminating Self-Generated Responses},
  author={Jiang, Dongwei and Zhang, Jingyu and Weller, Orion and Weir, Nathaniel and Van Durme, Benjamin and Khashabi, Daniel},
  booktitle={Proceedings of the AAAI Conference on Artificial Intelligence},
  volume={39},
  pages={24266--24275},
  year={2025},
  doi={10.1609/aaai.v39i23.34603}
}

@inproceedings{qi2026generalizationgap,
  title={On the Generalization Gap in Self-Evolving Language Model Reasoning},
  author={Qi, Zhenting and Baby, Susanna Maria and Baby, Stefanie Anna and Yuan, Kan and Tomkins, Andrew and Vu, Tu and Juan, Da-Cheng and Rashtchian, Cyrus},
  booktitle={Forty-third International Conference on Machine Learning},
  year={2026},
  url={https://openreview.net/forum?id=mnUidYi5qO}
}

@inproceedings{mahbub2026selfpreference,
  title={Mitigating Self-Preference by Authorship Obfuscation},
  author={Mahbub, Taslim and Feng, Shi},
  booktitle={Proceedings of the AAAI Conference on Artificial Intelligence},
  volume={40},
  pages={37701--37708},
  year={2026},
  doi={10.1609/aaai.v40i44.41105}
}

@inproceedings{kumar2025score,
  title={Training Language Models to Self-Correct via Reinforcement Learning},
  author={Kumar, Aviral and Zhuang, Vincent and Agarwal, Rishabh and Su, Yi and Co-Reyes, JD and Singh, Avi and Baumli, Kate and Iqbal, Shariq and Bishop, Colton and Roelofs, Rebecca and Zhang, Lei and McKinney, Kay and Shrivastava, Disha and Paduraru, Cosmin and Tucker, George and Precup, Doina and Behbahani, Feryal and Faust, Aleksandra},
  booktitle={International Conference on Learning Representations},
  year={2025},
  url={https://proceedings.iclr.cc/paper_files/paper/2025/hash/871ac99fdc5282d0301934d23945ebaa-Abstract-Conference.html}
}

@inproceedings{jiang2025importance,
  title={Importance Weighting Can Help Large Language Models Self-Improve},
  author={Jiang, Chunyang and Chan, Chi-Min and Xue, Wei and Liu, Qifeng and Guo, Yike},
  booktitle={Proceedings of the AAAI Conference on Artificial Intelligence},
  volume={39},
  pages={24257--24265},
  year={2025},
  doi={10.1609/aaai.v39i23.34602}
}

@inproceedings{zhou2026confabulations,
  title={Can LLMs Detect Their Confabulations? Estimating Reliability in Uncertainty-Aware Language Models},
  author={Zhou, Tianyi and Medina, Johanne and Chawla, Sanjay},
  booktitle={Proceedings of the AAAI Conference on Artificial Intelligence},
  volume={40},
  pages={38164--38172},
  year={2026},
  doi={10.1609/aaai.v40i44.41155}
}

@inproceedings{ge2025samule,
  title={{SAMULE}: Self-Learning Agents Enhanced by Multi-level Reflection},
  author={Ge, Yubin and Romeo, Salvatore and Cai, Jason and Sunkara, Monica and Zhang, Yi},
  booktitle={Proceedings of the 2025 Conference on Empirical Methods in Natural Language Processing},
  pages={16591--16610},
  year={2025},
  address={Suzhou, China},
  publisher={Association for Computational Linguistics},
  doi={10.18653/v1/2025.emnlp-main.839},
  url={https://aclanthology.org/2025.emnlp-main.839/}
}
\bibliographystyle{iclr2025_conference}

\appendix
\section{Theoretical Analysis}
\label{app:theory}
This section formalizes three properties of \textsc{Seed} that mirror the requirements identified in Section~\ref{sec1_intro}: the hindsight supervision is \emph{on-policy}, \emph{dense}, and \emph{self-evolving}. Specifically, we show that (i) synchronized on-policy hindsight induces an occupancy-matched adaptive target, (ii) the OPD term provides decision-specific credit even when outcome-based group advantages are uninformative, and (iii) refreshing the shared analyzer controls the staleness of the auxiliary gradient as the policy changes. These are local statements about the structure and freshness of the \textsc{Seed} update; they do not by themselves imply monotonic return improvement, which additionally requires the generated hindsight skills to be behaviorally informative.

\paragraph{Notation.}
Consider outer iteration $k$. At the beginning of the iteration, the current checkpoint is frozen as the behavior policy
$\pi_k=\pi_{\theta_k}$
and also instantiates the trajectory analyzer $\mathcal A_k$. Let $C$ denote an ordinary autoregressive context at a valid action-token position, and let $d_k(c)$ be the normalized token-context occupancy induced by $\pi_k$ and the environment. Conditional on $C=c$, the realized rollout token $Y$ is sampled as
\[
Y\sim\pi_k(\cdot\mid c).
\]
After the complete trajectory $\tau$ is observed, the synchronized analyzer produces the hindsight skill
$S=\mathcal A_k(x_\tau)$.
We overload $H(c,S)$ to denote the skill-augmented version of $c$, with the same action-token prefix as in the ordinary context. At update initialization, define the detached skill gate
\begin{equation}
 g_k(c,v,S)
 =
 \sigma\!\left(
 \beta_{\mathrm{opd}}
 \left[
 \log\pi_k(v\mid H(c,S))
 -
 \log\pi_k(v\mid c)
 \right]
 \right).
 \label{eq:theory_realized_gate}
\end{equation}
Importantly, $S$ is generated from the \emph{complete} trajectory and can therefore depend on the realized token $Y$, the subsequent rollout, environment feedback, and analyzer decoding. To preserve this dependence, we define the conditional expected gate
\begin{equation}
 w_k(c,v)
 =
 \mathbb E\!\left[
 g_k(C,Y,S)
 \mid C=c,\,Y=v
 \right],
 \label{eq:conditional_expected_gate}
\end{equation}
where the expectation is over the future trajectory, environment randomness, and any analyzer randomness. Standard softmax policies have full support, and the sigmoid lies strictly between zero and one, so $0<w_k(c,v)<1$. We state the results at $\theta=\theta_k$, where the behavior, analyzer, teacher branch, and student branch are synchronized. The adaptive-target identity below extends to an inner optimization step by replacing Eq.~\ref{eq:theory_realized_gate} with its current detached numerical value while retaining the same behavior occupancy $d_k$.

\subsection{On-Policy Hindsight Produces an Occupancy-Matched Adaptive Target}
\label{app:theory_onpolicy}
The first result characterizes the expected OPD direction and makes explicit why the source trajectories should be generated by the current policy. Rather than imitating every sampled token equally, \textsc{Seed} reweights the current policy according to how strongly its own hindsight skill supports each action.

\begin{proposition}[Occupancy-matched hindsight target]
\label{prop:onpolicy_target}
For every context $c$, define
\begin{equation}
 Z_k(c)
 =
 \sum_{v\in\mathcal V}
 \pi_k(v\mid c)w_k(c,v),
 \qquad
 r_k(v\mid c)
 =
 \frac{\pi_k(v\mid c)w_k(c,v)}{Z_k(c)}.
 \label{eq:onpolicy_target}
\end{equation}
Then $0<Z_k(c)<1$, $r_k(\cdot\mid c)$ is a probability distribution, and the expected OPD gradient at update initialization satisfies
\begin{equation}
 \left.
 \nabla_\theta\mathcal L_{\mathrm{opd},k}(\theta)
 \right|_{\theta=\theta_k}
 =
 \mathbb E_{c\sim d_k}
 \left[
 Z_k(c)
 \left.
 \nabla_\theta
 D_{\mathrm{KL}}
 \left(
 r_k(\cdot\mid c)
 \,\middle\|\,
 \pi_\theta(\cdot\mid c)
 \right)
 \right|_{\theta=\theta_k}
 \right].
 \label{eq:onpolicy_target_gradient}
\end{equation}
Thus, the OPD update is a Monte Carlo estimate of distillation toward a skill-reweighted target on the current policy's own token-context occupancy.

Moreover, let $\mu_k$ be the joint distribution of a completed trajectory and one of its valid token samples, and let $\phi_{k,\theta}(x)$ be the detached per-token OPD gradient obtained by analyzing sample $x$ with the current analyzer $\mathcal A_k$. If
$\|\phi_{k,\theta}(x)\|_2\le G$
for all relevant $x$, then replacing current on-policy samples by samples from an earlier distribution $\mu_j$ incurs
\begin{equation}
 \left\|
 \mathbb E_{x\sim\mu_k}[\phi_{k,\theta}(x)]
 -
 \mathbb E_{x\sim\mu_j}[\phi_{k,\theta}(x)]
 \right\|_2
 \le
 2G\,\operatorname{TV}(\mu_k,\mu_j).
 \label{eq:occupancy_mismatch_tv}
\end{equation}
If $\mu_k$ is absolutely continuous with respect to $\mu_j$, Pinsker's inequality further gives
\begin{equation}
 \left\|
 \mathbb E_{\mu_k}[\phi_{k,\theta}]
 -
 \mathbb E_{\mu_j}[\phi_{k,\theta}]
 \right\|_2
 \le
 G\sqrt{2D_{\mathrm{KL}}(\mu_k\|\mu_j)}.
 \label{eq:occupancy_mismatch_kl}
\end{equation}
\end{proposition}

\begin{proof}
Because $0<w_k(c,v)<1$ and $\pi_k(\cdot\mid c)$ is normalized, Eq.~\ref{eq:onpolicy_target} gives $0<Z_k(c)<1$ and $\sum_v r_k(v\mid c)=1$. In Eq.~\ref{eq:opd_objective}, both the teacher log-probability and the gate are detached, so only the ordinary student log-probability contributes to the gradient. Conditioning first on $C=c$ and $Y=v$, and then using Eq.~\ref{eq:conditional_expected_gate}, yields
\begin{align*}
 &\left.
 \nabla_\theta\mathcal L_{\mathrm{opd},k}(\theta)
 \right|_{\theta=\theta_k}
 \\
 &\quad=
 -\mathbb E_{c\sim d_k}
 \left[
 \sum_{v\in\mathcal V}
 \pi_k(v\mid c)w_k(c,v)
 \left.
 \nabla_\theta\log\pi_\theta(v\mid c)
 \right|_{\theta=\theta_k}
 \right]
 \\
 &\quad=
 -\mathbb E_{c\sim d_k}
 \left[
 Z_k(c)
 \sum_{v\in\mathcal V}
 r_k(v\mid c)
 \left.
 \nabla_\theta\log\pi_\theta(v\mid c)
 \right|_{\theta=\theta_k}
 \right],
\end{align*}
which is Eq.~\ref{eq:onpolicy_target_gradient}, since the entropy of the detached target $r_k$ has zero gradient.

For Eq.~\ref{eq:occupancy_mismatch_tv}, use the dual representation of the Euclidean norm and the variational characterization of total variation. For any unit vector $u$,
\[
 \left|
 \mathbb E_{\mu_k}[u^\top\phi_{k,\theta}]
 -
 \mathbb E_{\mu_j}[u^\top\phi_{k,\theta}]
 \right|
 \le
 2G\,\operatorname{TV}(\mu_k,\mu_j),
\]
because $|u^\top\phi_{k,\theta}(x)|\le G$. Taking the supremum over $u$ proves Eq.~\ref{eq:occupancy_mismatch_tv}. Equation~\ref{eq:occupancy_mismatch_kl} follows from
$\operatorname{TV}(\mu_k,\mu_j)\le\sqrt{D_{\mathrm{KL}}(\mu_k\|\mu_j)/2}$.
\end{proof}

Equation~\ref{eq:onpolicy_target} separates the two roles of the first \textsc{Seed} requirement. The factor $d_k(c)$ makes supervision \emph{occupancy matched}: it focuses learning on contexts, actions, and failure modes actually visited by the current policy. Within each such context, $w_k(c,v)$ makes supervision \emph{skill selective}: actions that are more strongly supported by the trajectory-specific hindsight skill receive greater relative mass. In particular,
\[
 \frac{r_k(u\mid c)}{\pi_k(u\mid c)}
 >
 \frac{r_k(v\mid c)}{\pi_k(v\mid c)}
 \quad\Longleftrightarrow\quad
 w_k(c,u)>w_k(c,v).
\]
The mismatch bounds show the complementary point: even if the same current analyzer is applied to old trajectories, a stale behavior distribution can bias the auxiliary direction in proportion to its divergence from the current trajectory-token distribution. On-policy collection sets this data-distribution mismatch to zero at the rollout stage.

The target also admits a useful value interpretation. Let
$Q_k(c,v)$
be the expected task return after choosing token $v$ at context $c$ and following the current policy thereafter. Then
\begin{equation}
 \mathbb E_{v\sim r_k(\cdot\mid c)}[Q_k(c,v)]
 -
 \mathbb E_{v\sim\pi_k(\cdot\mid c)}[Q_k(c,v)]
 =
 \frac{\operatorname{Cov}_{v\sim\pi_k(\cdot\mid c)}\!\left(Q_k(c,v),w_k(c,v)\right)}{Z_k(c)}.
 \label{eq:value_alignment_identity}
\end{equation}
Therefore, whenever hindsight support is positively correlated with current-policy action value, the reweighted target has higher local expected value than the unweighted policy. This condition makes explicit what the theorem does and does not assume: \textsc{Seed} converts hindsight support into an on-policy target, while the empirical benefit depends on the generated skills assigning greater support to better decisions.

\subsection{Dense Skill Credit Remains Informative under Sparse or Tied Rewards}
\label{app:theory_dense}
Outcome-based RL assigns the same trajectory-level advantage to every valid action token in a rollout. Consequently, it cannot distinguish a locally useful decision from a harmful one within the same trajectory, and its reward-driven signal disappears completely when all outcomes in a group are tied. The next result shows that the skill-conditioned OPD term can remain non-degenerate in exactly this regime.

\begin{proposition}[Decision-specific signal under reward ties]
\label{prop:dense_credit}
Consider a rollout group in which every trajectory has the same outcome. Then
$A^{\mathrm{rl}}_{q,n,t,\ell}=0$
for every valid token, and the clipped reward-driven term in Eq.~\ref{eq:rl_objective} has zero gradient. Fix a current-policy context $c$, parameterize the ordinary student distribution as
$p_z(\cdot\mid c)=\operatorname{softmax}(z(c))$,
and define the conditional expected OPD loss up to detached constants by
\begin{equation}
 \overline{\mathcal L}_{\mathrm{opd},k,c}(z)
 =
 -\sum_{v\in\mathcal V}
 \pi_k(v\mid c)w_k(c,v)
 \log p_z(v\mid c).
 \label{eq:conditional_opd_loss}
\end{equation}
At update initialization, where $p_z(\cdot\mid c)=\pi_k(\cdot\mid c)$,
\begin{equation}
 \frac{\partial
 \overline{\mathcal L}_{\mathrm{opd},k,c}}
 {\partial z(c,v)}
 =
 \pi_k(v\mid c)
 \left[Z_k(c)-w_k(c,v)\right].
 \label{eq:dense_credit_gradient}
\end{equation}
Furthermore, with the inverse-policy weighted norm
\[
 \|a\|^2_{\pi_k^{-1}}
 =
 \sum_{v\in\mathcal V}
 \frac{a(v)^2}{\pi_k(v\mid c)},
\]
the gradient magnitude satisfies the exact identity
\begin{equation}
 \left\|
 \nabla_{z(c)}
 \overline{\mathcal L}_{\mathrm{opd},k,c}
 \right\|^2_{\pi_k^{-1}}
 =
 \operatorname{Var}_{v\sim\pi_k(\cdot\mid c)}
 \left[w_k(c,v)\right].
 \label{eq:gate_variance_identity}
\end{equation}
Hence the OPD gradient is nonzero if and only if the expected hindsight gate is non-constant over candidate tokens with positive current-policy probability.
\end{proposition}

\begin{proof}
If all group outcomes are identical, then
$R(\tau_q^{(n)})-\mu_q=0$
for every rollout, so the group-relative advantage in Eq.~\ref{eq:rl_objective} is zero. The clipped surrogate therefore contributes no reward-derived policy gradient; the separate KL regularizer may still contribute a gradient, but it does not provide task-outcome credit.

For the OPD term, the softmax derivative gives
\[
 \frac{\partial[-\log p_z(u\mid c)]}{\partial z(c,v)}
 =
 p_z(v\mid c)-\mathbf 1\{u=v\}.
\]
Applying this identity to Eq.~\ref{eq:conditional_opd_loss} and setting $p_z=\pi_k$ yields
\begin{align*}
 \frac{\partial
 \overline{\mathcal L}_{\mathrm{opd},k,c}}
 {\partial z(c,v)}
 &=
 \left(
 \sum_u\pi_k(u\mid c)w_k(c,u)
 \right)\pi_k(v\mid c)
 -
 \pi_k(v\mid c)w_k(c,v)
 \\
 &=
 \pi_k(v\mid c)
 \left[Z_k(c)-w_k(c,v)\right],
\end{align*}
which proves Eq.~\ref{eq:dense_credit_gradient}. Substituting this gradient into the weighted norm gives
\begin{align*}
 \left\|
 \nabla_{z(c)}
 \overline{\mathcal L}_{\mathrm{opd},k,c}
 \right\|^2_{\pi_k^{-1}}
 &=
 \sum_v
 \pi_k(v\mid c)
 \left[w_k(c,v)-Z_k(c)\right]^2
 \\
 &=
 \operatorname{Var}_{v\sim\pi_k(\cdot\mid c)}
 \left[w_k(c,v)\right],
\end{align*}
because $Z_k(c)=\mathbb E_{v\sim\pi_k}[w_k(c,v)]$. The variance is zero exactly when $w_k(c,v)$ is constant $\pi_k$-almost surely.
\end{proof}

Equations~\ref{eq:dense_credit_gradient}--\ref{eq:gate_variance_identity} formalize the dense credit supplied by \textsc{Seed}. Tokens with
$w_k(c,v)>Z_k(c)$
receive a negative logit derivative and are promoted by gradient descent, whereas tokens with
$w_k(c,v)<Z_k(c)$
are relatively suppressed. Thus, although every gate value is nonnegative, softmax normalization converts variation in hindsight support into signed relative credit over candidate decisions. The variance in Eq.~\ref{eq:gate_variance_identity} quantifies the strength of this token-level signal: the more the hindsight skill discriminates among alternatives, the larger the OPD update can be. This establishes \emph{informativeness}, not automatic correctness; whether the signal favors better actions is characterized by the value-alignment covariance in Eq.~\ref{eq:value_alignment_identity}.

The tied-reward case is the sharpest contrast, but the distinction is more general. GRPO broadcasts one scalar
$A^{\mathrm{rl}}_{q,n}$
across all valid tokens in a trajectory, while \textsc{Seed} assigns a context- and token-dependent coefficient through $w_k(c,v)$. It can therefore preserve useful local behaviors in failed trajectories and attenuate locally inefficient choices in successful trajectories, even when the terminal reward alone cannot identify those decisions.

\subsection{Self-Evolving Synchronization Controls Analyzer Staleness}
\label{app:theory_selfevolving}
The final requirement concerns the source of the hindsight itself. A fixed analyzer or one-time skill dataset may be informative early in training but can become mismatched to the capabilities, strategies, and failure modes of a later policy. We quantify this mismatch by its effect on the detached OPD gate and hence on the auxiliary gradient.

For a completed token sample $x=(\tau,c,v)$ from the current distribution $\mu_k$, let
$s_r(x)=\mathcal A_r(x_\tau)$
be the skill produced by an analyzer checkpoint from iteration $r$. Holding the current policy $\pi_k$ fixed for this comparison, define
\begin{equation}
 \Delta_{k,r}(x)
 =
 \log\pi_k(v\mid H(c,s_r(x)))
 -
 \log\pi_k(v\mid c),
 \qquad
 g_{k,r}(x)
 =
 \sigma\!\left(
 \beta_{\mathrm{opd}}\Delta_{k,r}(x)
 \right).
 \label{eq:analyzer_shift_gate}
\end{equation}
The first index identifies the policy whose probabilities are re-scored, while the second identifies the analyzer that supplies the skill. For a fixed trainable parameter $\theta$, define the resulting expected auxiliary gradient on current trajectories as
\begin{equation}
 U_k(\mathcal A_r;\theta)
 =
 -\mathbb E_{x\sim\mu_k}
 \left[
 g_{k,r}(x)
 \nabla_\theta\log\pi_\theta(v\mid c)
 \right].
 \label{eq:analyzer_gradient_definition}
\end{equation}

\begin{proposition}[Analyzer-staleness bound]
\label{prop:selfevolving_staleness}
Assume
$\|\nabla_\theta\log\pi_\theta(v\mid c)\|_2\le G$
for all relevant current-policy token samples. Then the discrepancy between the OPD gradient induced by the synchronized analyzer $\mathcal A_k$ and that induced by an earlier analyzer $\mathcal A_j$ satisfies
\begin{equation}
 \begin{aligned}
 \left\|
 U_k(\mathcal A_k;\theta)
 -
 U_k(\mathcal A_j;\theta)
 \right\|_2
 &\le
 G\,
 \mathbb E_{x\sim\mu_k}
 \left[
 |g_{k,k}(x)-g_{k,j}(x)|
 \right]
 \\
 &\le
 \frac{\beta_{\mathrm{opd}}G}{4}
 \mathbb E_{x\sim\mu_k}
 \left[
 |\Delta_{k,k}(x)-\Delta_{k,j}(x)|
 \right].
 \end{aligned}
 \label{eq:analyzer_staleness_bound}
\end{equation}
In particular, using the analyzer instantiated from the current checkpoint sets this cross-iteration analyzer mismatch to zero at the rollout-and-analysis stage.
\end{proposition}

\begin{proof}
Subtract Eq.~\ref{eq:analyzer_gradient_definition} for $r=k$ and $r=j$, apply Jensen's inequality, and use the score-norm bound:
\begin{align*}
 &\left\|
 U_k(\mathcal A_k;\theta)
 -
 U_k(\mathcal A_j;\theta)
 \right\|_2
 \\
 &\quad\le
 \mathbb E_{x\sim\mu_k}
 \left[
 |g_{k,k}(x)-g_{k,j}(x)|
 \left\|
 \nabla_\theta\log\pi_\theta(v\mid c)
 \right\|_2
 \right]
 \\
 &\quad\le
 G\,
 \mathbb E_{x\sim\mu_k}
 \left[
 |g_{k,k}(x)-g_{k,j}(x)|
 \right].
\end{align*}
Because
$\sup_z\sigma'(z)=1/4$,
the map $z\mapsto\sigma(\beta_{\mathrm{opd}}z)$ is
$\beta_{\mathrm{opd}}/4$-Lipschitz. Applying this fact pointwise to Eq.~\ref{eq:analyzer_shift_gate} proves the second inequality.
\end{proof}

Proposition~\ref{prop:selfevolving_staleness} measures analyzer staleness in the quantity that directly affects learning: the change in skill-induced log-probability shift on the current policy's own trajectories. It does not require parameter distance between analyzers to be meaningful, nor does it assume that the updated analyzer is always more accurate. Instead, it states that if an old and a current analyzer produce skills with different behavioral implications, their OPD gradients can differ accordingly. The factor $\beta_{\mathrm{opd}}/4$ also exposes a trade-off: a sharper gate more strongly separates supported and unsupported tokens, but is correspondingly more sensitive to stale skill-induced shifts.

Combining the data and analyzer effects gives a compact decomposition of the two synchronization mechanisms. At a fixed current policy $\pi_k$ and trainable parameter $\theta$, let
$U_{k,\theta}(\mu,\mathcal A_r)$
denote Eq.~\ref{eq:analyzer_gradient_definition} with samples drawn from $\mu$. Under the same score-norm bound,
\begin{equation}
 \begin{aligned}
 &\left\|
 U_{k,\theta}(\mu_k,\mathcal A_k)
 -
 U_{k,\theta}(\mu_j,\mathcal A_j)
 \right\|_2
 \\
 &\qquad\le
 2G\,\operatorname{TV}(\mu_k,\mu_j)
 +
 \frac{\beta_{\mathrm{opd}}G}{4}
 \mathbb E_{x\sim\mu_j}
 \left[
 |\Delta_{k,k}(x)-\Delta_{k,j}(x)|
 \right].
 \end{aligned}
 \label{eq:joint_staleness_decomposition}
\end{equation}
The first term is \emph{experience staleness}, controlled by collecting trajectories on-policy; the second is \emph{supervision staleness}, controlled by refreshing the analyzer from the latest shared checkpoint. At each outer iteration, \textsc{Seed} uses the synchronized pair $(\mu_k,\mathcal A_k)$, so both cross-iteration mismatch terms are reset at data generation. During the subsequent inner policy optimization, the analyzer remains fixed at $\mathcal A_k$; therefore, the claim is not that analyzer lag is identically zero at every optimizer step, but that any within-update lag is discarded when trajectories and skills are regenerated from the next policy checkpoint.

Taken together, Propositions~\ref{prop:onpolicy_target},~\ref{prop:dense_credit}, and~\ref{prop:selfevolving_staleness} provide a one-to-one theoretical counterpart to the three design requirements in Section~\ref{sec1_intro}. On-policy collection aligns the auxiliary objective with the current policy's visited distribution, dense skill-conditioned re-scoring supplies token-specific credit beyond terminal outcomes, and the self-evolving shared analyzer prevents the hindsight target from remaining permanently tied to an earlier policy checkpoint.

\section{Additional Experimental Details}\label{app:experimental_details}
This section provides additional experimental details, including the datasets, baseline methods, the complete \textsc{Seed} algorithm and representative extracted skills, and the implementation details of our method.

\subsection{Datasets}
\label{app:datasets}

Our experiments cover three representative forms of long-horizon agentic interaction: embodied household reasoning, web navigation, and search-augmented question answering.
Table~\ref{tab:appendix_datasets_summary} summarizes the benchmarks and the data sizes used for SFT, RL training, and evaluation.


\begin{table}[ht!]
\centering
\caption{Detailed information on the agentic benchmarks and training data.}
\label{tab:appendix_datasets_summary}
\vspace{0.04in}

\small
\setlength{\tabcolsep}{4pt}
\renewcommand{\arraystretch}{1.12}
\renewcommand{\tabularxcolumn}[1]{m{#1}}

\begin{tabularx}{\linewidth}{
    @{}
    >{\raggedright\arraybackslash}p{0.18\linewidth}
    >{\raggedright\arraybackslash}X
    >{\centering\arraybackslash}p{0.15\linewidth}
    >{\centering\arraybackslash}p{0.15\linewidth}
    >{\centering\arraybackslash}p{0.15\linewidth}
    @{}
}
\toprule
\textbf{Domain}
& \textbf{Benchmark(s)}
& \makecell{\textbf{\#SFT Train}\\\textbf{Samples}}
& \makecell{\textbf{\#RL Train}\\\textbf{Samples}}
& \makecell{\textbf{\#Test}\\\textbf{Samples}} \\
\midrule

\makecell[l]{Embodied\\Reasoning}
& ALFWorld
& 180
& 2,400
& \makecell{140 (seen)\\134 (unseen)} \\

\addlinespace[0.45em]
\rowcolor{gray!10}
\makecell[l]{Web\\Navigation}
& WebShop
& 180
& 2,400
& 128 \\

\addlinespace[0.45em]

\makecell[l]{Search-Augmented\\QA}
& \makecell[l]{
    NQ $\cdot$ TriviaQA $\cdot$ PopQA\\
    HotpotQA $\cdot$ 2WikiMultiHopQA\\
    MuSiQue $\cdot$ Bamboogle
}
& 180
& 19,200
& 51,713 \\

\bottomrule
\end{tabularx}
\end{table}

\paragraph{ALFWorld.}
ALFWorld~\citep{shridhar2021alfworld} exposes household environments derived from ALFRED through a text-based interaction interface.
Given a natural-language instruction and a sequence of textual observations, the agent must select admissible actions to accomplish the specified household goal.
We evaluate six task categories: \textit{Pick}, \textit{Look}, \textit{Clean}, \textit{Heat}, \textit{Cool}, and \textit{Pick2}.
We report the main results on the 140-task seen split and use the 134-task unseen split for out-of-distribution evaluation.

\paragraph{WebShop.}
WebShop~\citep{yao2022webshop} evaluates agents in an interactive e-commerce environment.
For each user request, the agent searches the product catalog, examines candidate items, selects the required attributes, and attempts to complete a purchase satisfying the specified constraints.
The benchmark provides a normalized task-completion score for partial constraint satisfaction and a binary success metric for exact completion.
We evaluate all methods on 128 test samples.

\paragraph{Search-Augmented QA.}
Following the Search-R1 protocol~\citep{jin2025searchr1}, we construct the search-based QA setting from seven datasets:
Natural Questions~\citep{kwiatkowski2019natural},
TriviaQA~\citep{joshi2017triviaqa},
PopQA~\citep{mallen2023popqa},
HotpotQA~\citep{yang2018hotpotqa},
2WikiMultiHopQA~\citep{ho2020constructing},
MuSiQue~\citep{trivedi2022musique}, and
Bamboogle~\citep{press2023measuring}.
The agent may issue search queries, inspect retrieved documents, and synthesize the collected evidence before returning its final answer.
The combined evaluation set contains 51,713 questions.

\paragraph{Training data configuration.}
For the SFT stage, we select 180 tasks from the training split of each benchmark setting.
Each selected task is executed with 8 independent rollouts, yielding 1,440 trajectories for constructing the corresponding trajectory--skill training set.
The subsequent RL stage uses 2,400 training instances for ALFWorld, 2,400 for WebShop, and 19,200 for Search-Augmented QA.
SFT data construction and RL optimization are performed separately for each benchmark configuration.

\subsection{Baselines}
\label{app:baselines}

We compare \textsc{Seed} against three categories of baselines: prompting-only methods, outcome-based reinforcement learning methods, and self-distillation or skill-distillation methods.
Unless a method is marked with an asterisk, evaluation is conducted using only the standard task prompt and the interaction history returned by the environment.
The superscript $^{*}$ indicates that a natural-language skill is additionally supplied during validation or testing, while the backbone model and evaluation setting remain unchanged.

\paragraph{Prompting-only methods.}
\begin{itemize}
    \item \textit{Vanilla}.
    We directly evaluate the original instruction-tuned backbone without performing any additional post-training.
    The agent receives the default environment prompt together with the observation--action history accumulated during interaction.

    \item \textit{Skill-Prompt}$^{*}$.
    This baseline uses the same frozen parameters as \textit{Vanilla}, but appends a retrieved task-relevant skill to the model context during validation and testing.
    Since no parameter optimization is involved, this comparison measures the benefit of using natural-language skills purely as inference-time guidance.
\end{itemize}

\paragraph{Outcome-based reinforcement learning.}
\begin{itemize}
    \item \textit{GRPO}~\citep{shao2024deepseekmath}.
    GRPO is a critic-free reinforcement learning algorithm that samples multiple trajectories for each task and normalizes their scalar outcome rewards within the rollout group to obtain relative advantages.
    Each valid token in a trajectory is optimized using the corresponding sequence-level advantage and a clipped importance-ratio objective.
    In our implementation, GRPO relies solely on terminal environment rewards and does not use process-level annotations, teacher predictions, or skill-conditioned supervision.

    \item \textit{Skill-GRPO}.
    This variant retains the standard GRPO optimization objective while introducing a task-relevant natural-language skill into the policy context during both rollout collection and parameter updates.
    The skill can therefore affect the policy's exploration behavior and the trajectories subsequently reinforced by the outcome reward.
    It is removed during validation and testing, allowing us to assess whether the training-time guidance has been internalized by the policy.

    \item \textit{Skill-GRPO}$^{*}$.
    This baseline follows the same training procedure as \textit{Skill-GRPO}, but continues to provide the skill during validation and testing.
    It therefore maintains consistent skill conditioning across training and evaluation.
\end{itemize}

\paragraph{Self-distillation and skill-distillation methods.}
\begin{itemize}
    \item \textit{OPSD}~\citep{zhao2026selfdistilledreasoner}.
    OPSD derives student and teacher branches from the same underlying model while conditioning them on different contexts.
    The student generates trajectories from the ordinary task context, whereas the teacher additionally observes privileged information available only during training.
    The teacher then re-scores the student's sampled tokens and produces dense token-level targets through distribution matching.
    Teacher-side outputs are detached from gradient computation, and the privileged context is not used at inference time.

    \item \textit{GRPO+OPSD}.
    This baseline jointly optimizes the trajectory-level GRPO loss and the token-level OPSD objective.
    Environment rewards provide outcome-based supervision for complete trajectories, while OPSD contributes dense guidance at individual token positions.
    It serves as a controlled comparison for determining whether a straightforward combination of outcome-based RL and generic on-policy self-distillation is sufficient.

    \item \textit{Skill-SD}~\citep{wang2026skillsd}.
    Skill-SD adapts self-distillation to multi-turn agent training by representing completed experience as compact natural-language skills.
    During optimization, a retrieved skill is supplied only to the teacher branch, while the student continues to operate from the ordinary task context.
    The student is therefore trained to absorb the behavioral guidance conveyed by the skill-conditioned teacher without requiring explicit skills during evaluation.

    \item \textit{RLSD}~\citep{yang2026rlsd}.
    RLSD employs a privileged self-teacher to refine token-level credit assignment within reinforcement learning.
    Rather than introducing an independent distribution-matching objective, it converts the teacher--student log-probability gap into a bounded coefficient that adjusts the magnitude of each token's GRPO update.
    The sign of the update remains determined by the environment-derived advantage, meaning that the privileged teacher controls update strength but not the reinforcement direction.
    The self-distillation contribution is emphasized early in training and gradually reduced toward the standard GRPO objective.

    \item \textit{SDAR}~\citep{lu2026sdar}.
    SDAR preserves GRPO as the main outcome-optimization objective and adds a separately gated self-distillation loss.
    A privileged teacher branch evaluates the student's on-policy tokens under an augmented context, while a bounded gate controls the contribution of each teacher signal.
    The gate may depend on student uncertainty and the detached teacher--student log-probability difference, strengthening reliable positive guidance and reducing the influence of potentially noisy signals.
    In contrast to \textit{RLSD}, SDAR leaves the original GRPO advantage unchanged and applies its gating mechanism only to the auxiliary distillation term.
\end{itemize}

For all reproduced post-training baselines, we use the same backbone models and environment interfaces as \textsc{Seed}.
Whenever supported by the corresponding method, we also align the task batch size, rollout budget, rollout group size, number of policy updates, and evaluation protocol.
The resulting comparisons primarily differ in their optimization objectives and in whether skills or other privileged information are available during training or evaluation.

\subsection{Algorithm and Extracted Skill Examples}\label{app:algorithm_skills}
Algorithm~\ref{alg:seed} summarizes the self-evolving training loop. At each update, the frozen snapshot $\pi_{\theta_{\mathrm{old}}}$ both collects trajectories and analyzes the completed interactions. The current trainable policy $\pi_\theta$ then re-scores the same sampled tokens under ordinary and skill-augmented contexts; gradients flow only through the ordinary branch. The resulting OPD loss is optimized jointly with the KL-regularized GRPO objective before the updated policy becomes the next snapshot.

Table~\ref{tab:episode_level_skills} provides representative skills extracted from successful and failed trajectories across ALFWorld, WebShop, and Search-based QA.

\begin{algorithm}[t]
\caption{\textsc{Seed}: Self-Evolving On-Policy Distillation}
\label{alg:seed}
\footnotesize
\begin{algorithmic}[1]
\Require SFT-initialized policy $\pi_{\theta_{\mathrm{sft}}}$, task set $\mathcal{Q}$,
fixed reference policy $\pi_{\mathrm{ref}}$, context function $H$, group size $N$,
gate sharpness $\beta_{\mathrm{opd}}$, KL coefficient $\beta_{\mathrm{KL}}$,
OPD coefficient $\lambda_{\mathrm{opd}}$, clip parameter $\epsilon_{\mathrm{clip}}$,
learning rate $\eta$
\Ensure Trained policy $\pi_\theta$
\State $\theta \gets \theta_{\mathrm{sft}}$
\For{each policy update}
    \State $\theta_{\mathrm{old}} \gets \theta$
    \State Sample a task batch $\mathcal{B} \subset \mathcal{Q}$
    \State \textcolor{green!50!black}{\textit{// On-policy experience and synchronized skill analysis}}
    \For{each task $q \in \mathcal{B}$}
        \State Sample $\mathcal{G}_q=\{\tau_q^{(n)}\}_{n=1}^{N}$ with
        $\tau_q^{(n)}\sim\pi_{\theta_{\mathrm{old}}}(\cdot\mid q)$
        \State Compute $\mu_q$ and $\sigma_q$ from
        $\{R(\tau_q^{(n)})\}_{n=1}^{N}$
        \For{each trajectory $\tau_q^{(n)}\in\mathcal{G}_q$}
            \State $A^{\mathrm{rl}}_{q,n}\gets
            (R(\tau_q^{(n)})-\mu_q)/(\sigma_q+\epsilon)$
            \State $s_q^{(n)}\gets
            A_{\theta_{\mathrm{old}}}(x_{\tau_q^{(n)}})$
            \For{each action step $t$ in $\tau_q^{(n)}$}
                \State $\tilde h_{q,n,t}\gets
                H(h_{q,n,t},s_q^{(n)})$
                \State Cache $\ell^{\mathrm{old}}_{q,n,t,\ell}$ for all
                valid sampled tokens $\ell$
            \EndFor
        \EndFor
    \EndFor
    
    \State \textcolor{green!50!black}{\textit{// Paired contextual re-scoring and joint optimization}}
    \For{each optimization minibatch of valid sampled tokens}
        \State Evaluate $\ell^{\mathrm{skill}}_{q,n,t,\ell}$ and
        $\ell^{\theta}_{q,n,t,\ell}$ using the current $\pi_\theta$
        \State $\Delta_{q,n,t,\ell}\gets
        \operatorname{sg}[\ell^{\mathrm{skill}}_{q,n,t,\ell}
        -\ell^{\theta}_{q,n,t,\ell}]$
        \State $g_{q,n,t,\ell}\gets
        \sigma(\beta_{\mathrm{opd}}\Delta_{q,n,t,\ell})$
        \State $\rho_{q,n,t,\ell}(\theta)\gets
        \exp(\ell^{\theta}_{q,n,t,\ell}
        -\ell^{\mathrm{old}}_{q,n,t,\ell})$
        \State Compute the clipped GRPO loss $\mathcal{L}_{\mathrm{rl}}$
        from $\rho$, $A^{\mathrm{rl}}$, and $\epsilon_{\mathrm{clip}}$
        \State $\mathcal{L}_{\mathrm{opd}}\gets
        \mathbb{E}[m\cdot g\cdot (\operatorname{sg}[\ell^{\mathrm{skill}}]
        -\ell^{\theta})]$
        \State $\mathcal{L}_{\mathrm{SEED}}\gets
        \mathcal{L}_{\mathrm{rl}}
        +\lambda_{\mathrm{opd}}\mathcal{L}_{\mathrm{opd}}$
        \State $\theta\gets\theta-\eta\nabla_\theta
        \mathcal{L}_{\mathrm{SEED}}$
    \EndFor
\EndFor
\end{algorithmic}
\end{algorithm}


\begin{table*}[ht!]
\centering
\small
\setlength{\tabcolsep}{4pt}
\renewcommand{\arraystretch}{1.15}
\caption{\textbf{Episode-level skills extracted by the trajectory analyzer.}
For each dataset, we present one successful and one failed trajectory.
Skills derived from successful episodes capture reusable workflows, whereas
those derived from failed episodes distill actionable failure-avoidance rules.}
\label{tab:episode_level_skills}
\vspace{0.1in}

\begin{tabularx}{\textwidth}{
  >{\raggedright\arraybackslash}p{0.10\linewidth}
  >{\raggedright\arraybackslash}p{0.31\linewidth}
  >{\raggedright\arraybackslash}X
}
\toprule
\multicolumn{1}{c}{\textbf{Outcome}} &
\multicolumn{1}{c}{\textbf{Task}} &
\multicolumn{1}{c}{\textbf{Episode-level skill}} \\
\midrule

\rowcolor{gray!15}
\multicolumn{3}{l}{\textbf{ALFWorld}} \\

Success &
clean some spatula and put it in diningtable. &
Workflow: When tasked to clean and place an object, first locate the target
object, take it to a cleaning station, clean it, and finally move it to the
target location. \\

Failure &
put a clean spatula in drawer. &
Avoid moving an object to the target location before verifying both inventory
and required object state. Confirm the object is held and already satisfies
required conditions such as cleanliness; avoid repeatedly moving objects
without checking state or placement success. \\

\addlinespace
\rowcolor{gray!15}
\multicolumn{3}{l}{\textbf{WebShop}} \\

Success &
Find me women's jumpsuits, rompers \& overalls with button closure, quality
polyester, polyester spandex, long sleeve for daily wear with color: hot pink,
and size: x-large, and price lower than \$40.00 dollars. &
Workflow: Use a structured search query covering all required attributes,
choose a result matching the core product category, then verify and select
required options such as color and size before clicking Buy Now. \\

Failure &
Find me machine wash men's dress shirts with polyester heathers, heathers
cotton, cotton heather, needle sleeve, classic fit with color: navy, and fit
type: youth, and size: 3x-large, and price lower than \$50.00 dollars. &
Avoid clicking irrelevant search results or purchasing partial matches. If
results do not match the target product type, reformulate the query; on the
product page, select each required attribute once and verify material, color,
size, fit, and price before Buy Now. \\

\addlinespace
\rowcolor{gray!15}
\multicolumn{3}{l}{\textbf{Search}} \\

Success &
Tie a Yellow Ribbon is the third album by American popular music group Dawn
(Michael Anthony Orlando Cassavitas, Telma Hopkins \& Joyce Vincent Wilson)
released in which year? &
Workflow: For a specific factual query, search the exact query, then extract
and verify the answer directly from the search results before responding. \\

Failure &
Le Juge is a comic whose story is inspired by a man who called himself what? &
Avoid prematurely inferring an answer from search results that do not directly
address the core query. First verify the relevant entity/relation, then extract
the requested attribute. \\

\bottomrule
\end{tabularx}
\end{table*}

\subsection{Implementation Details}
\label{app:implementation_details}

\paragraph{Benchmark metrics.}
We adopt benchmark-specific measures following the corresponding evaluation protocols.
For ALFWorld, let $\mathrm{SR}_{c}$ denote the success rate on task category $c$.
The overall result assigns equal weight to all six categories and is computed as
\begin{equation}
    \mathrm{ALFWorld\text{-}Avg}
    =
    \frac{1}{6}
    \sum_{c=1}^{6}
    \mathrm{SR}_{c}.
\end{equation}

For Search-based QA, we use a dataset-balanced aggregate rather than pooling all questions together.
Denoting the accuracy on dataset $d$ by $\mathrm{Acc}_{d}$, the reported score is
\begin{equation}
    \mathrm{Search\text{-}Avg}
    =
    \frac{1}{7}
    \sum_{d=1}^{7}
    \mathrm{Acc}_{d}.
\end{equation}

For WebShop, we report two complementary metrics.
\textit{Score} measures the degree of task completion by reflecting how fully the purchased product satisfies the requirements specified in the user request; the normalized scores are averaged and scaled by 100.
\textit{Succ.} denotes the percentage of episodes in which all requirements are satisfied and the task is completed exactly.

\paragraph{External analyzer at the SFT stage.}
To construct the trajectory--skill supervision for SFT, we serialize each completed rollout into a structured record comprising the task instruction, the full interaction trajectory, and the terminal outcome.
The external analyzer then examines the completed episode and produces a natural-language hindsight skill.
The analyzer is used exclusively for offline skill annotation and does not participate in trajectory collection, which is carried out by the corresponding backbone model itself.
We instantiate the external analyzer with GLM-5.2~\citep{zai2026glm52}, set the temperature to 0.0, and limit the maximum response length to 4,096 tokens.
The complete prompt used for hindsight-skill annotation is provided in Figure~\ref{fig:analyzer_prompt}.

\paragraph{Actor and analyzer prompts at the RL stage.}
During Stage~2, the actor follows the standard environment-interaction prompt for the corresponding benchmark; the prompt used for ALFWorld is shown in Figure~\ref{fig:actor_prompt}.
The analyzer uses the same trajectory-analysis prompt as in the SFT stage, as presented in Figure~\ref{fig:analyzer_prompt}.

\paragraph{Training hyperparameters.}
Table~\ref{tab:rl_hyperparameters} records the training hyperparameters that are required for exact reproduction.

\begin{table*}[t]
\centering
\small
\setlength{\tabcolsep}{6pt}
\renewcommand{\arraystretch}{1.15}
\caption{Training hyperparameters for \textsc{Seed}.}
\label{tab:rl_hyperparameters}
\vspace{0.1in}
\begin{tabular}{@{}p{0.4\textwidth}p{0.5\textwidth}@{}}
\toprule
\textbf{Hyperparameter} & \textbf{Value} \\
\midrule

Training steps
& 150  \\

Training batch size
& 16 for ALFWorld and WebShop; 128 for Search \\

Rollout group size $N$
& 8 \\

Learning rate
& $1\times10^{-6}$  \\

PPO clip parameter $\epsilon_\mathrm{clip}$
& 0.2 \\

Sharpness of OPD gate $\beta_\mathrm{opd}$
& 5.0 \\
 
OPD loss coefficient $\lambda_{\mathrm{opd}}$
& 0.01 \\

KL coefficient $\beta_{\mathrm{KL}}$
& 0.01 \\

Maximum prompt length
& 2,048 for ALFWorld ; 4,096 for WebShop and Search  \\

Response length
& 512 \\

Maximum interaction steps
& 30 for ALFWorld, 15 for WebShop, and 4 for Search. \\
\bottomrule
\end{tabular}
\end{table*}

\paragraph{Computing details.}
Training is conducted on 8 Nvidia A800 80G GPUs.

\section{Supplementary Results}\label{app:supp_results}

\subsection{Detailed Sample Efficiency Comparison}
\label{app:sample_efficiency}
Table~\ref{tab:sample_efficiency} provides the full comparison across five training data fractions on ALFWorld and WebShop. \textsc{Seed} outperforms GRPO at every fraction on both benchmarks. On ALFWorld, the gains range from 13.4 to 30.2 points. With only 60\% of the data, \textsc{Seed} reaches 80.7, exceeding the 75.0 achieved by GRPO with the full dataset. On WebShop, \textsc{Seed} improves performance by 5.5 to 15.6 points and reaches 75.0 with 80\% of the data, compared with 63.3 for GRPO trained on the full dataset. These results show that dense hindsight supervision makes more effective use of collected trajectories and remains beneficial as the training set grows.


\begin{table}[htbp]
\centering
\caption{\textbf{Sample efficiency comparison on ALFWorld and WebShop.}
We report performance using different proportions of the available training data. SEED consistently outperforms GRPO across both benchmarks, while $\Delta$ denotes the absolute performance improvement over GRPO.}
\vspace{0.05in}
\label{tab:sample_efficiency}
\setlength{\tabcolsep}{8pt}
\renewcommand{\arraystretch}{1.08}

\begin{tabular}{lccccc}
\toprule
\textbf{Method}
& \textbf{20\%}
& \textbf{40\%}
& \textbf{60\%}
& \textbf{80\%}
& \textbf{100\%} \\
\midrule

\multicolumn{6}{l}{\textit{ALFWorld}} \\
GRPO
& 27.3 & 42.2 & 56.3 & 58.6 & 75.0 \\
\rowcolor{gray!10}
SEED
& 40.7 & 58.9 & 80.7 & 88.8 & 91.8 \\
\rowcolor{blue!6}
$\Delta$
& +13.4 & +16.7 & +24.4 & +30.2 & +16.8 \\

\midrule

\multicolumn{6}{l}{\textit{WebShop}} \\
GRPO
& 31.3 & 45.3 & 57.0 & 63.6 & 63.3 \\
\rowcolor{gray!10}
SEED
& 37.5 & 53.1 & 62.5 & 75.0 & 78.9 \\
\rowcolor{blue!6}
$\Delta$
& +6.2 & +7.8 & +5.5 & +11.4 & +15.6 \\

\bottomrule
\end{tabular}
\end{table}


\begin{table}[ht!]
\centering
\caption{\textbf{Cross-domain generalization results on ALFWorld Unseen.} Using Qwen2.5-3B-Instruct as the backbone, we evaluate performance on six unseen task categories. SEED outperforms GRPO on five categories and increases the average success rate by 15.3 percentage points.}
\vspace{0.05in}
\label{tab:ood_generalization_alfworld}
\small
\setlength{\tabcolsep}{9pt}
\renewcommand{\arraystretch}{1.08}
\begin{tabular}{cccccccc}
\toprule[1.2pt]
& \multicolumn{7}{c}{\textbf{ALFWorld Unseen}} \\
\cmidrule(lr){2-8}
\textbf{Method}
& \textbf{Pick}
& \textbf{Look}
& \textbf{Clean}
& \textbf{Heat}
& \textbf{Cool}
& \textbf{Pick2}
& \textbf{Avg.} \\
\midrule
ReAct
& 17.4
& 6.7
& 8.8
& 7.4
& 9.1
& 0.0
& 8.2 \\
GRPO
& 73.9
& 60.0
& \textbf{82.4}
& 59.3
& 72.7
& 76.9
& 70.9 \\
\rowcolor{gray!10}
SEED
& \textbf{90.4}
& \textbf{78.3}
& 79.5
& \textbf{94.3}
& \textbf{86.2}
& \textbf{88.2}
& \textbf{86.2} \\
\rowcolor{blue!6}
$\Delta$
& +16.5
& +18.3
& -2.9
& +35.0
& +13.5
& +11.3
& +15.3 \\
\bottomrule[1.2pt]
\end{tabular}
\end{table}

\subsection{Cross-Domain Generalization}
\label{app:ood_generalization}
Table~\ref{tab:ood_generalization_alfworld} reports results for each task family on ALFWorld Unseen. \textsc{Seed} improves the average success rate from 70.9 to 86.2 and outperforms GRPO on five of the six families. The largest gain occurs on \textit{Heat}, where the success rate increases by 35.0 points. Improvements are also substantial on \textit{Look} and \textit{Pick}, reaching 18.3 and 16.5 points, respectively. Although performance on \textit{Clean} decreases by 2.9 points, the broad gains across the remaining families indicate that \textsc{Seed} learns reusable behavioral guidance that transfers beyond the training environments.

\subsection{Multimodal Extension}
\label{app:multimodal_extension}
To examine whether \textsc{Seed} extends beyond text-only interaction, we evaluate Qwen2.5-VL-3B-Instruct~\citep{bai2025qwen2} on two vision-based agentic benchmarks. Sokoban~\citep{SchraderSokoban2018} presents each state as a $6\times6$ visual grid containing the player, boxes, walls, and target locations. The agent must navigate the grid and push every box onto a target. Since boxes cannot be pulled, a poor push may create an irreversible dead end. The task therefore tests visual state tracking and long-horizon spatial planning. EZPoints from Gym Cards~\citep{NEURIPS2024_c848b7d3} presents two playing cards together with a partially constructed expression. At each step, the agent selects a card value or an arithmetic operator to build an expression that evaluates to 12, using each card exactly once. This task couples fine-grained card recognition with sequential arithmetic reasoning.

As shown in Table~\ref{tab:vlm_success_rates}, \textsc{Seed} achieves success rates of 82.0\% on Sokoban and 100.0\% on EZPoints. It improves over GRPO by 14.9 and 13.1 points, respectively, raising the average success rate from 77.0\% to 91.0\%. ReAct reaches only 7.4\% on average, which highlights the need for policy learning in these visually grounded tasks. The consistent gains across spatial planning and visual arithmetic show that \textsc{Seed} can turn multimodal trajectories into useful hindsight supervision and internalize that guidance through self-evolving OPD, without skill prompts during evaluation. Figure~\ref{fig:sokoban_case} presents a step-by-step visualization of a representative Sokoban trajectory.

\begin{figure}[ht!]
    \centering
    \includegraphics[width=\textwidth]{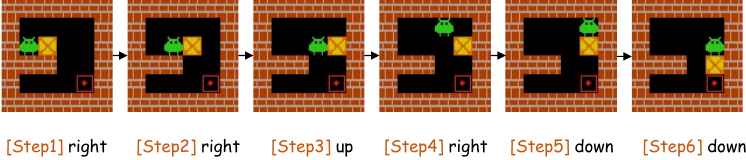}
    \setlength{\abovecaptionskip}{1pt}
    \caption{\textbf{A representative trajectory on Sokoban.} The sequence shows six consecutive actions executed by the agent. Arrows indicate the temporal progression of the trajectory, and the action taken at each step is displayed below the corresponding observation.}
    \label{fig:sokoban_case}
\end{figure}


\begin{table}[ht!]
\centering
\caption{\textbf{Results on vision-based agentic benchmarks}. We report
success rates (\%) on Sokoban [6$\times$6] and EZPoints using
Qwen2.5-VL-3B-Instruct as the backbone model.}
\label{tab:vlm_success_rates}
\vspace{0.05in}

\small
\setlength{\tabcolsep}{18pt}
\renewcommand{\arraystretch}{1.12}

\begin{tabular}{cccc}
\toprule
\textbf{Method}
& \textbf{Sokoban [6$\times$6] $\uparrow$}
& \textbf{EZPoints $\uparrow$}
& \textbf{Average $\uparrow$} \\
\midrule

ReAct
& 11.7
& 3.1
& 7.4 \\

GRPO
& 67.1
& 86.9
& 77.0 \\

\rowcolor{gray!12}
SEED
& \textbf{82.0}
& \textbf{100.0}
& \textbf{91.0} \\

\rowcolor{blue!6}
$\Delta$
& +14.9
& +13.1
& +14.0 \\

\bottomrule
\end{tabular}
\end{table}

\subsection{Additional Training Dynamics}
Figure~\ref{fig:success_rate_3x3} presents the success-rate trajectories for all three backbones on ALFWorld, Search-based QA, and WebShop. Performance improves in every setting, although the convergence pattern varies by domain. ALFWorld shows the largest absolute increase and approaches a success rate of 0.9 for all three backbones. Search-based QA improves rapidly during the early updates before stabilizing between 0.47 and 0.55. WebShop follows a steadier upward trend and finishes between 0.68 and 0.75. The consistent progress across model families and environments shows that \textsc{Seed} is not tied to a particular backbone or form of agentic interaction.

\begin{figure}[ht!]
    \centering
    \includegraphics[width=0.98\textwidth]{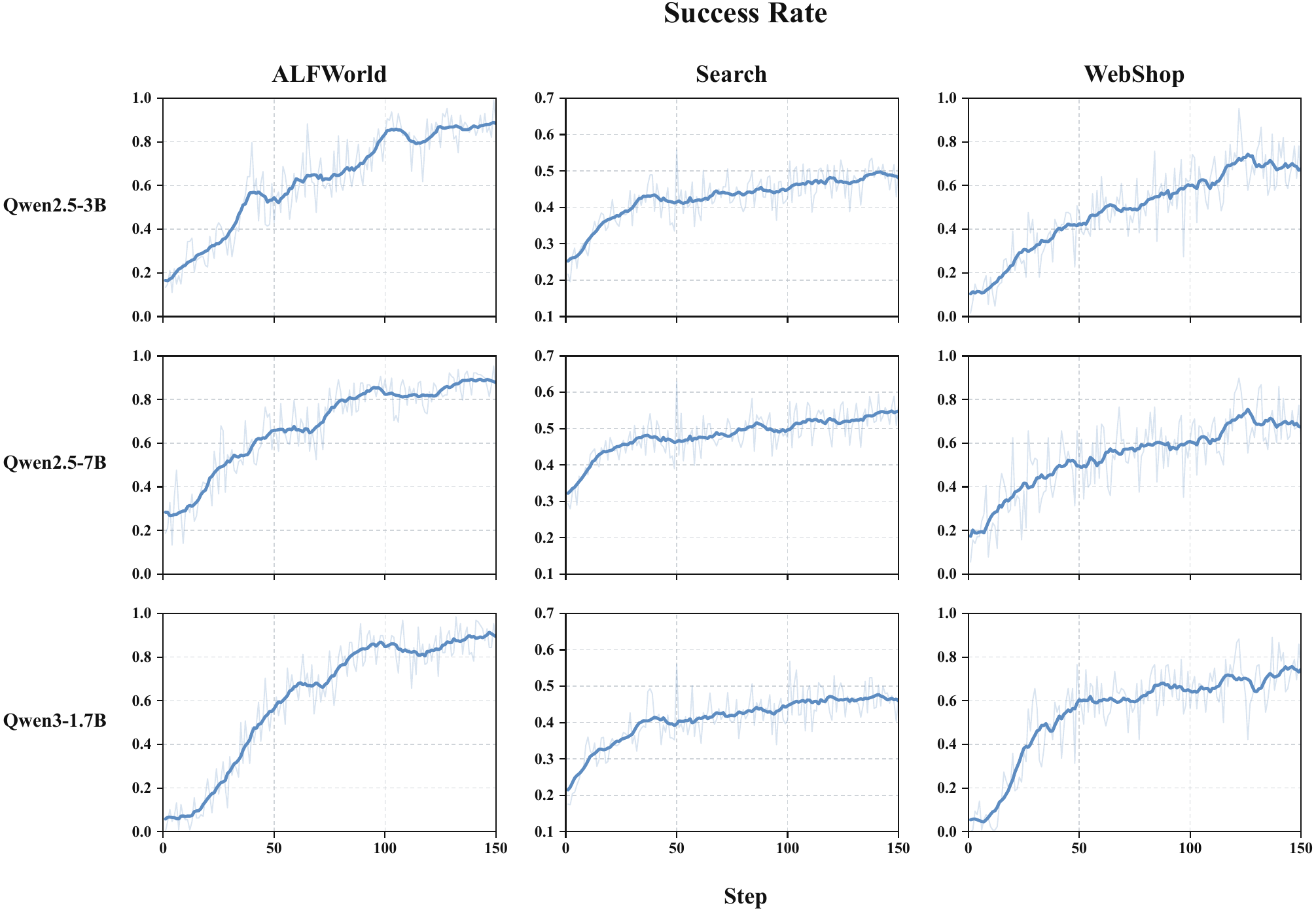}
    \caption{\textbf{Success rates across three backbones and three domains.} Success rates increase over training in all nine settings, showing consistent learning across model scales and agentic tasks.}
    \label{fig:success_rate_3x3}
\end{figure}
\begin{figure}[h!]
    \centering
    \includegraphics[width=0.98\textwidth]{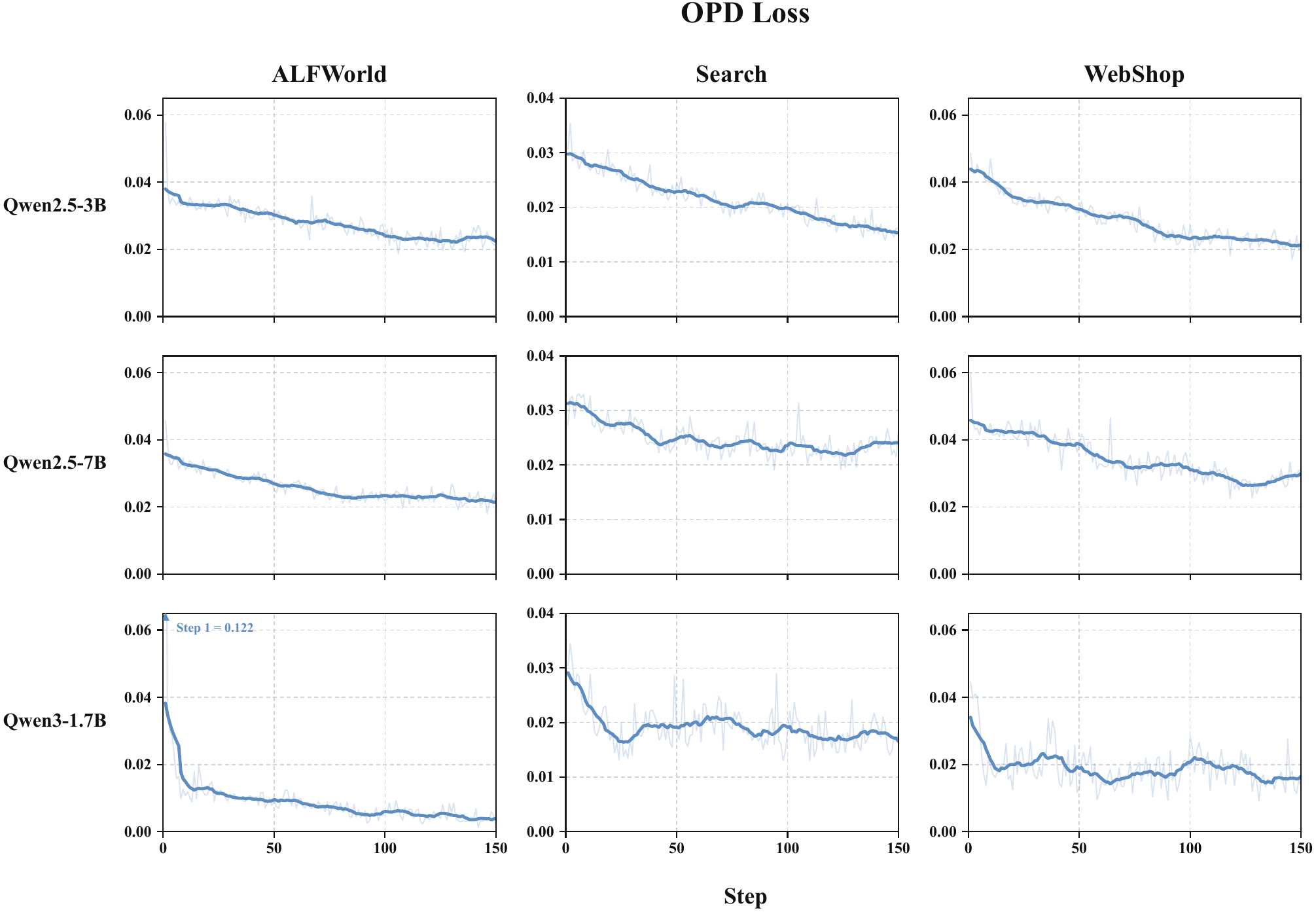}
    \caption{\textbf{OPD loss dynamics.} The loss generally decreases and stabilizes during training, indicating that the policy progressively internalizes the behavioral guidance provided by hindsight skills.}
    \label{fig:opd_loss_3x3}
\end{figure}

Figure~\ref{fig:opd_loss_3x3} shows the corresponding OPD losses. The loss decreases from its initial value and stabilizes at a lower level in all nine settings. The Qwen2.5 models exhibit gradual convergence, while Qwen3-1.7B shows a sharper early decline, especially on ALFWorld. Lower OPD loss indicates that the ordinary policy increasingly assigns probability to actions favored by hindsight supervision. Together with the rising success rates, this trend supports the stability of the self-evolving loop. The latest policy generates updated experience and hindsight skills, and OPD internalizes that guidance in subsequent updates.

\section{Case Study}\label{app:case_study}
Figures~\ref{fig:alfworld_case1}--\ref{fig:webshop_case2} illustrate how \textsc{Seed} applies internalized behavioral guidance during evaluation without skill inputs. In the first ALFWorld example, the policy decomposes the task into locating the ladle, cleaning it, opening the drawer, and completing the placement. It correctly handles the drawer as a necessary precondition before the final action. The second example requires two books and repeated navigation between the desk and bed. The policy tracks that the first book has already been placed, returns for the second, and completes the remaining subgoal without losing progress. These trajectories demonstrate coherent state tracking and precondition management over extended interactions.

The Search-based QA examples show that the policy adapts its information gathering to the available evidence. It answers the first question after one search because the retrieved passages directly establish the shared profession. For the second question, it first identifies \textit{Finding Neverland}, then issues a targeted search for its director before answering \textit{Marc Forster}. In WebShop, the policy retains the requested attributes and price limits from search through purchase. It verifies the product details and selects the required color and size before buying. Together, these cases show that \textsc{Seed} learns to decompose tasks, revise plans from new evidence, and preserve constraints over long trajectories. Because skills are not provided during inference, the observed behavior reflects guidance internalized through self-evolving OPD rather than external prompting.

\section{Additional Discussion}\label{app:discussion}
\textsc{Seed} treats completed on-policy experience as an evolving source of supervision. Our experiments cover several forms of long-horizon interaction, but broader benchmarks such as DeepPlanning~\citep{zhang2026deepplanning}, Long-Horizon-Terminal-Bench~\citep{li2026longhorizonterminal}, OdysseyArena~\citep{xu2026odysseyarena}, and RobotEQ~\citep{fang2026roboteq} provide more demanding tests. These settings involve longer workflows, richer state spaces, and greater interaction complexity, often with rare terminal success. Evaluating \textsc{Seed} in such environments would test whether policy-synchronized hindsight can preserve decisive events across extended contexts and remain useful as the policy explores more diverse behaviors. Benchmarks with intermediate or partial-credit grading also make it possible to study how external progress signals interact with the token-level hindsight supervision.

However, longer tasks also expose a central risk of self-evolution: internally generated supervision can inherit model errors and plateau below oracle-supervised training~\citep{jiang2025selfincorrect,qi2026generalizationgap}. Because the actor and analyzer share the same policy, improvements transfer across both roles, but their blind spots can also be shared. An inaccurate analysis may turn a recurring policy error into an apparently reusable rule, which later updates could reinforce. Evidence of self-preference in model-based evaluation further suggests that a model may systematically favor outputs resembling its own~\citep{mahbub2026selfpreference}. On-policy alignment and confidence gating help reduce distribution mismatch and noisy self-generated supervision~\citep{kumar2025score,jiang2025importance}, but neither confidence nor self-judgment guarantees semantic correctness~\citep{jiang2025selfincorrect,zhou2026confabulations}. A more self-correcting version of \textsc{Seed} could anchor skills to verifiable state changes, compare analyses across policy snapshots, and preserve uncertainty alongside each skill. It could also organize hindsight hierarchically by first summarizing local transitions and then deriving trajectory-level rules~\citep{ge2025samule}. Search-discovered reasoning abstractions~\citep{wu2024beyond} may provide useful structures for this aggregation. Policy-aware exploration~\citep{wu2026spark} could collect trajectories that distinguish competing rules, while uncertainty-aware self-distillation~\citep{lu2026sdar} may offer stronger filters for deciding which guidance to internalize.

Training efficiency is another practical constraint. \textsc{Seed} adds no deployment overhead, but training requires trajectory analysis and paired scoring under ordinary and skill-augmented contexts. The cost grows with interaction length and multimodal context. Speculative decoding methods such as \textsc{Double}~\citep{shen2026double} and \textsc{DSpark}~\citep{cheng2026dspark} could reduce the autoregressive cost of rollout collection and skill generation. Cached representations and batched paired scoring could further reduce the cost of OPD. Another promising direction is selective analysis: the analyzer could focus on trajectories with high uncertainty, novel states, or disagreement between reward and hindsight. This would preserve the self-evolving supervision loop while avoiding repeated analysis of redundant experience.
\newpage

\begin{figure}[t]
    \centering
    \includegraphics[width=0.99\textwidth]{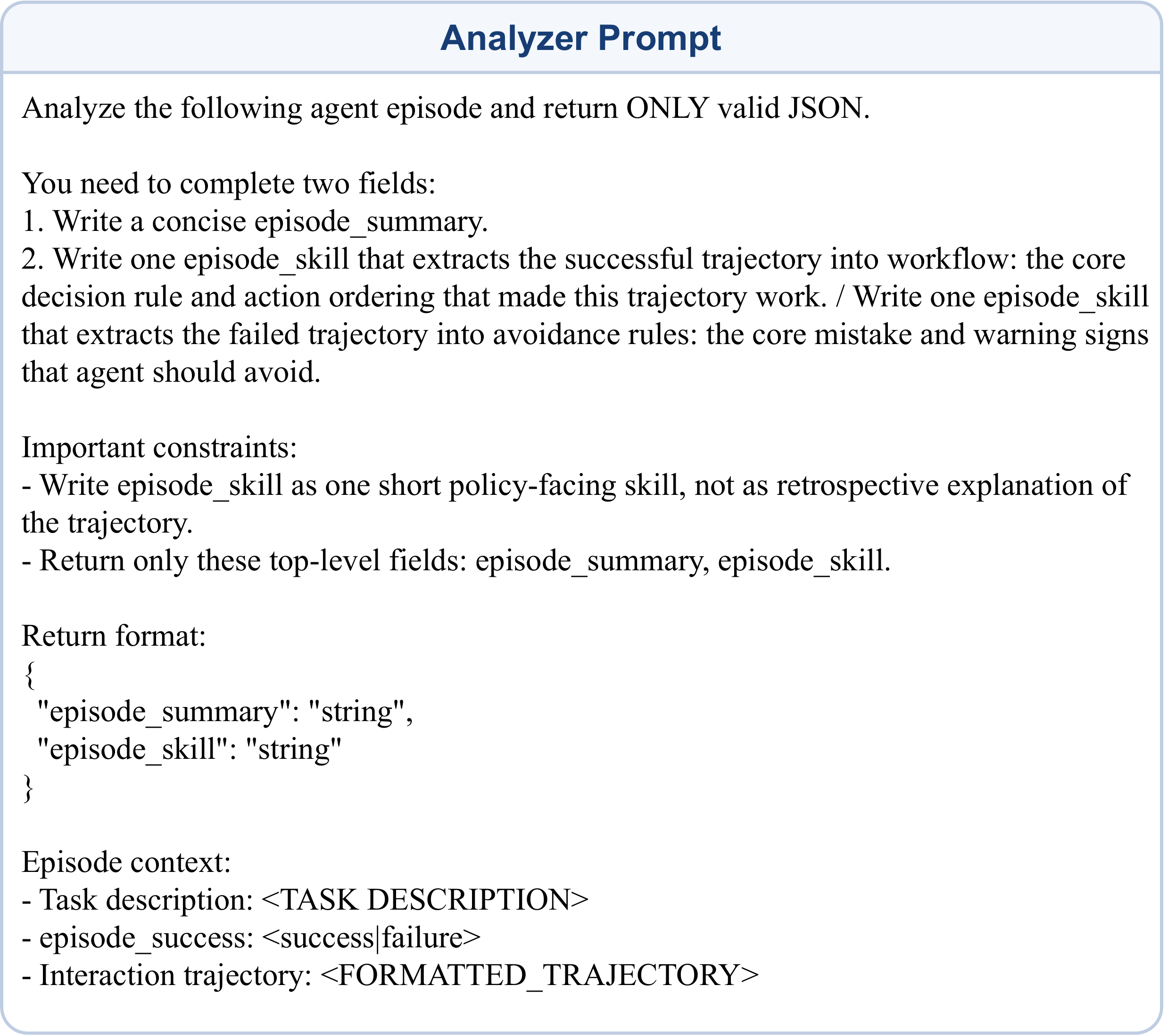}
    \caption{Prompt of analyzer.}
    \label{fig:analyzer_prompt}
\end{figure}

\begin{figure}[t]
    \centering
    \includegraphics[width=0.99\textwidth]{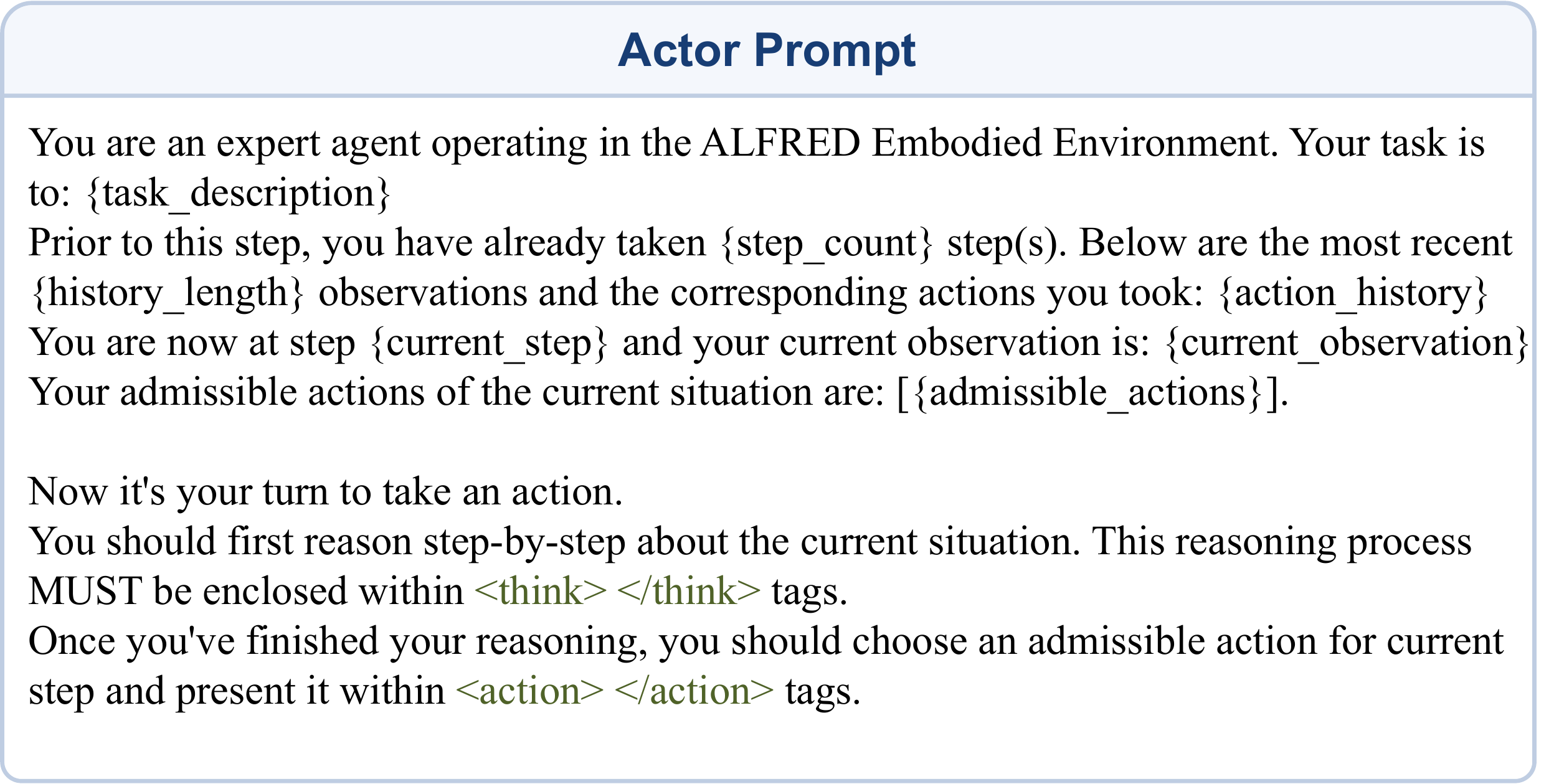}
    \caption{Prompt of actor (the policy model) in ALFWorld.}
    \label{fig:actor_prompt}
\end{figure}

\begin{figure}[hp]
    \centering
    \includegraphics[width=\textwidth]{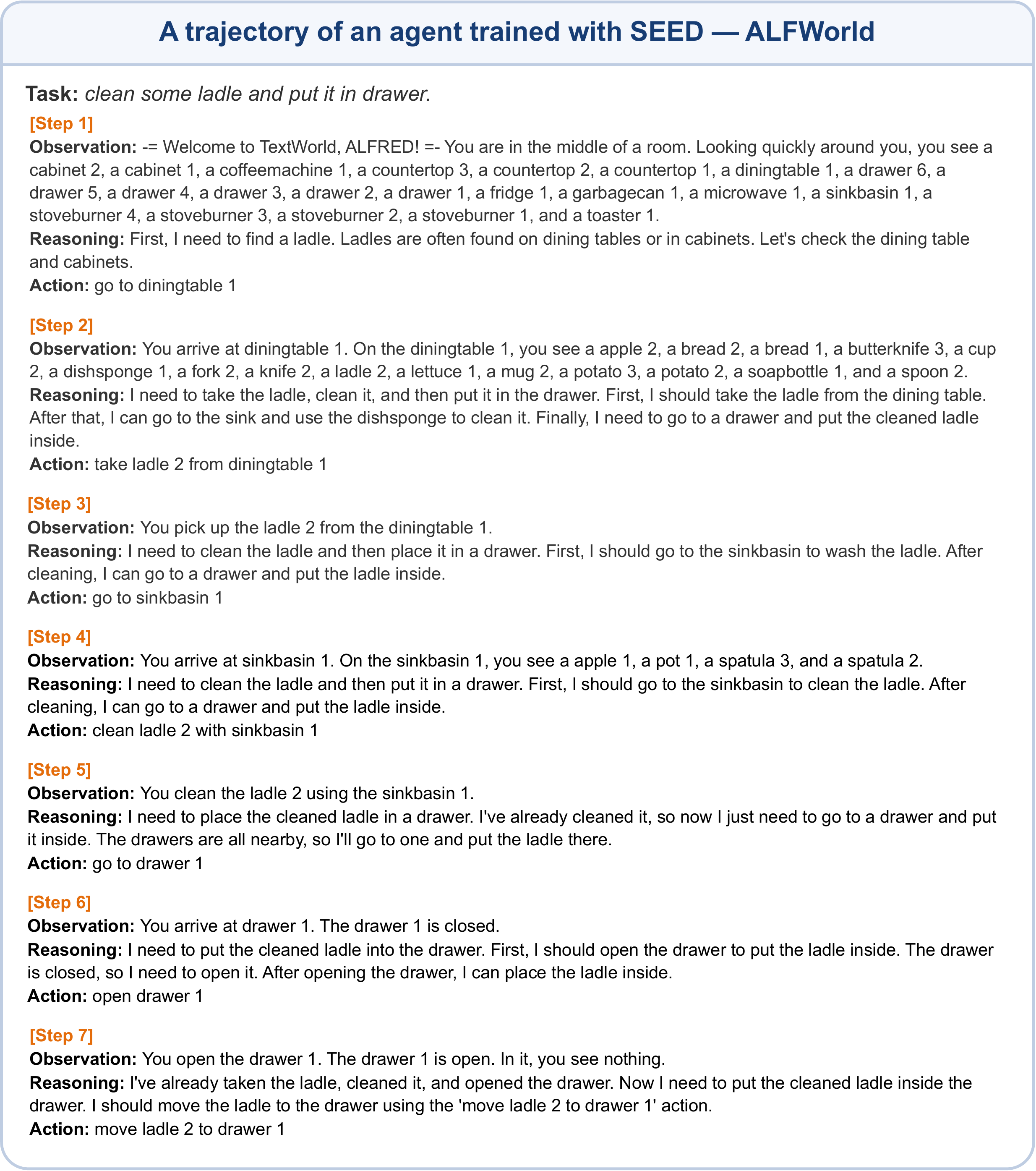}
    \caption{\textbf{ALFWorld Example 1.} \textsc{Seed} locates and cleans the ladle, opens the closed drawer, and completes the placement in seven steps, demonstrating task decomposition and correct precondition handling.}
    \label{fig:alfworld_case1}
\end{figure}
\begin{figure}[hp]
    \centering
    \includegraphics[width=\textwidth]{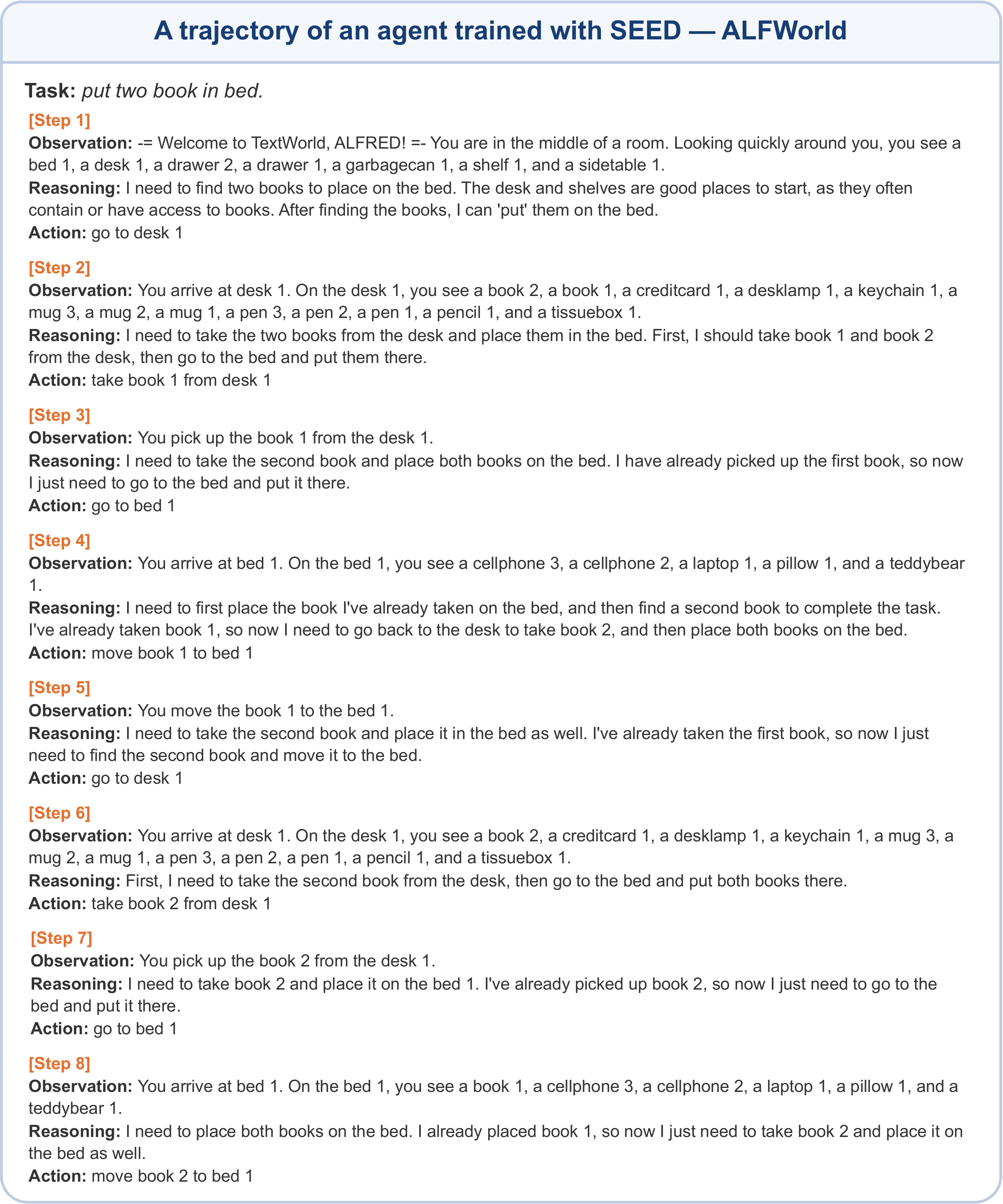}
    \caption{\textbf{ALFWorld Example 2.} \textsc{Seed} places two books through repeated trips between the desk and bed while retaining the completed first placement, demonstrating progress tracking over an extended trajectory.}
    \label{fig:alfworld_case2}
\end{figure}
\begin{figure}[hp]
    \centering
    \includegraphics[width=\textwidth]{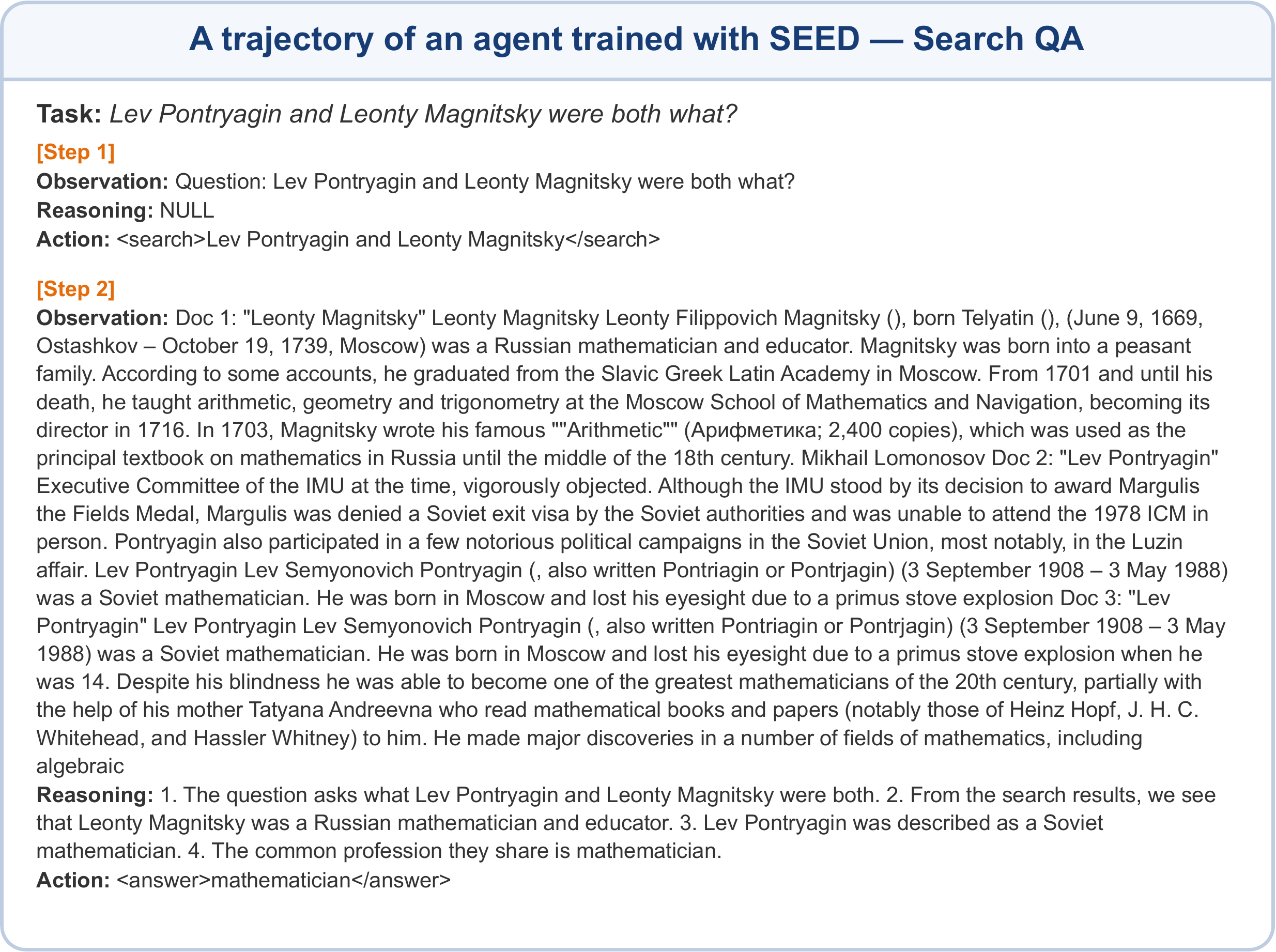}
    \caption{\textbf{Search-based QA Example 1.} A single query retrieves evidence that both people are mathematicians, allowing \textsc{Seed} to answer without unnecessary searches.}
    \label{fig:searchqa_case1}
\end{figure}
\begin{figure}[hp]
    \centering
    \includegraphics[width=\textwidth]{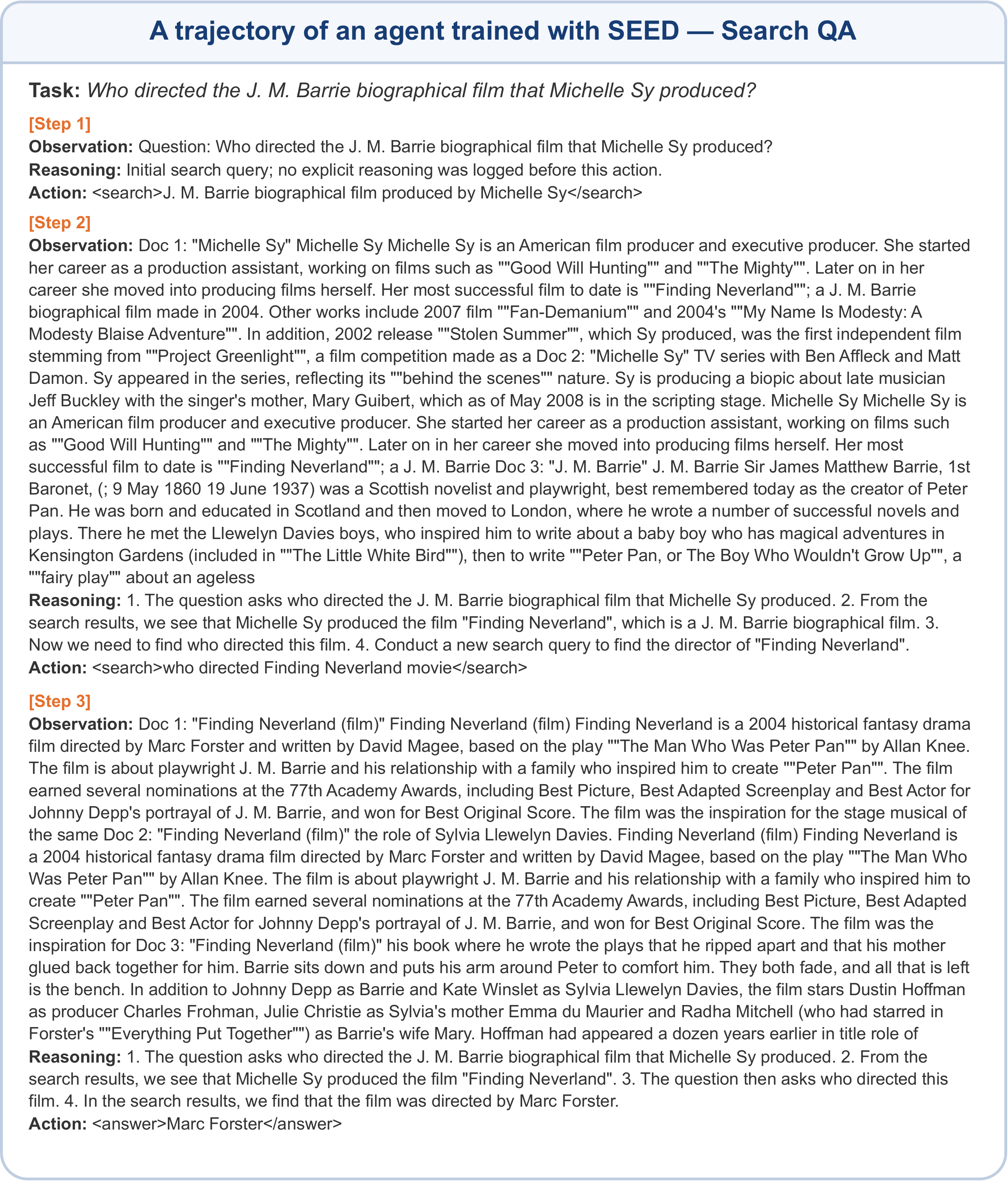}
    \caption{\textbf{Search-based QA Example 2.} \textsc{Seed} first identifies \textit{Finding Neverland}, then performs a targeted second search to verify that Marc Forster directed the film.}
    \label{fig:searchqa_case2}
\end{figure}
\begin{figure}[hp]
    \centering
    \includegraphics[width=\textwidth]{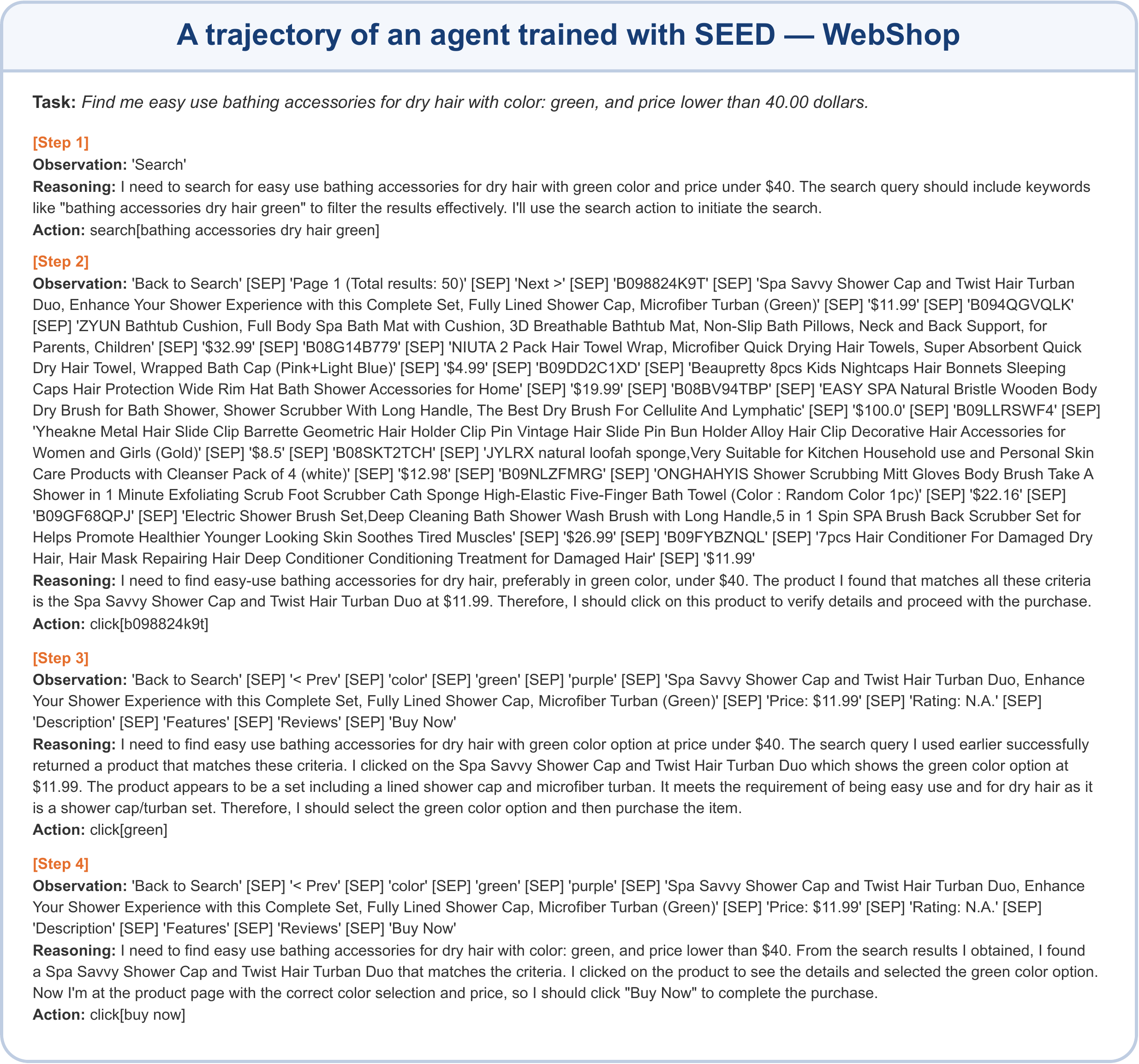}
    \caption{\textbf{WebShop Example 1.} \textsc{Seed} identifies a green shower cap and hair turban set under \$40, verifies the selected variant, and completes the purchase in four steps.}
    \label{fig:webshop_case1}
\end{figure}
\begin{figure}[hp]
    \centering
    \includegraphics[width=\textwidth]{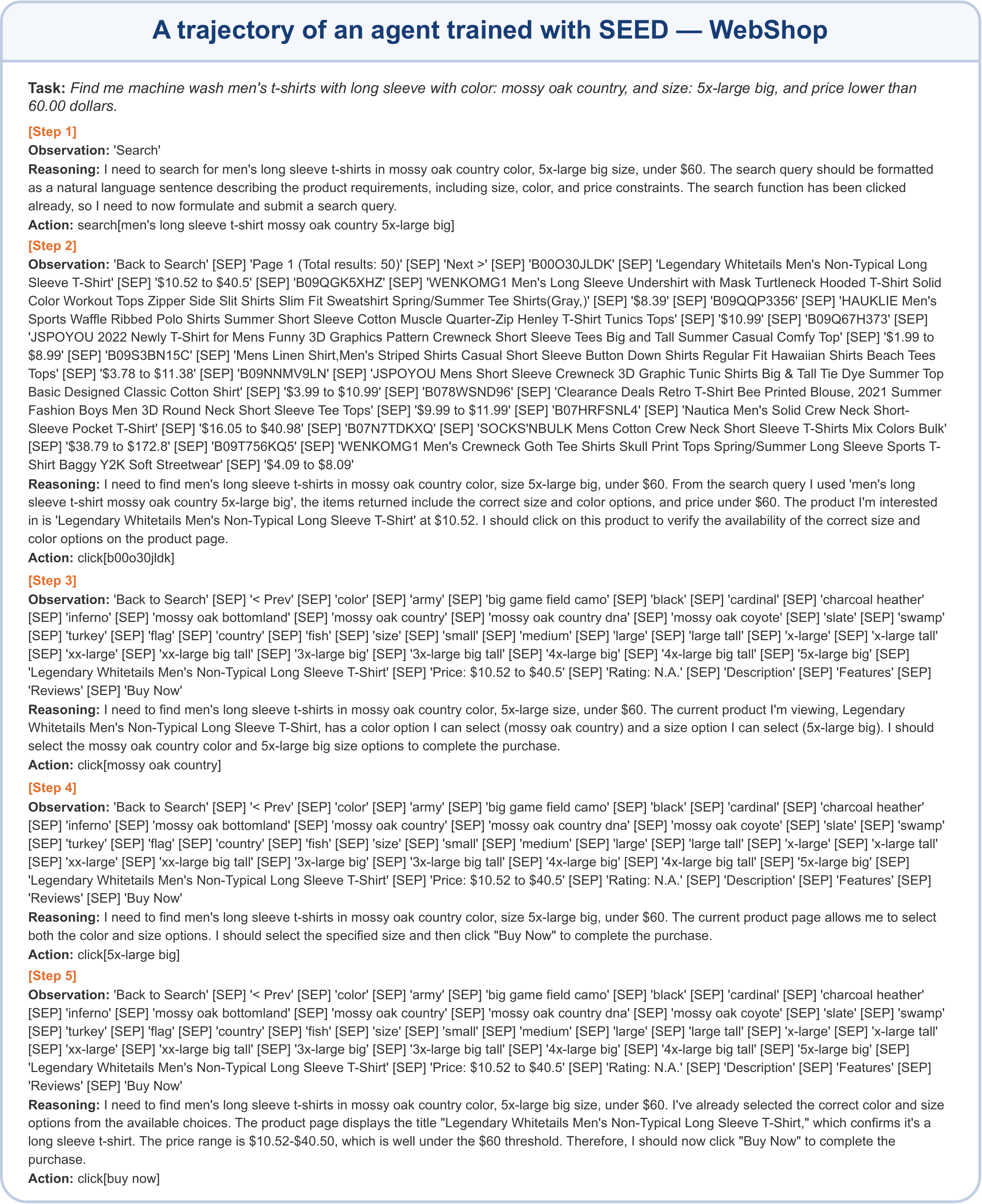}
    \caption{\textbf{WebShop Example 2.} \textsc{Seed} preserves the requested product constraints and selects \textit{mossy oak country} and \textit{5x-large big} before purchasing a matching long-sleeve shirt under \$60.}
    \label{fig:webshop_case2}
\end{figure}

\end{document}